\documentclass[smallextended,envcountsect]{svjour3}

\renewcommand{\subclassname}{%
  {\bfseries Mathematics Subject Classification (2020)}\enspace
}

\usepackage{graphicx}
\journalname{JOTA}
\usepackage{amsmath}
\usepackage{hyperref}
\usepackage{algorithm}
\usepackage{algpseudocode}
\usepackage{amssymb}    
\usepackage{xcolor}     

\begin{document}

\title{Smooth Reparameterizations of Functions on Simplicial Product Spaces: Applications to Probabilistic Tensor Decomposition and Functional Data Registration}
\author{
  Shashwat Kumar \and
  Arafat Rahman \and
  Anuj Srivastava \and
  P.-A. Absil
}

\institute{
  Shashwat Kumar \at
  Department of Biomedical Engineering, Johns Hopkins University, USA \\
  \email{skuma118@jh.edu}
  \and
  Arafat Rahman \at
  Department of Systems and Information Engineering, University of Virginia, USA \\
  \email{jgh6ds@virginia.edu}
  \and
  Anuj Srivastava \at
  Department of Applied Mathematics and Statistics, Johns Hopkins University, USA\\
  \email{anuj.srivastava@jhu.edu}
  \and
  P.-A. Absil \at
  ICTEAM Institute, UCLouvain, 1348 Louvain-la-Neuve, Belgium \\
  \email{pa.absil@uclouvain.be }
}

\date{Received: date / Accepted: date}

\maketitle
\begin{abstract}
We consider optimization problems defined on product spaces of simplices. Examples of this class of problems include learning low-rank discrete multivariate probability distributions via simplex constrained tensor decomposition and performing functional data registration under the Square Root Velocity Function (SRVF) representation. In this work, we demonstrate the feasibility of replacing the product simplex with a smooth, elementwise strictly convex reparameterization, resulting in an unconstrained optimization problem on a manifold. We show that performing such a reparameterization results in the second order Karush–Kuhn–Tucker (KKT) points on the smooth manifold being mapped to the weak second order KKT points on the product simplex. This leads to a Riemannian Gradient Descent (RGD) algorithm for solving the reparameterized problem, which outperforms Projected Gradient Descent (PGD), and provides a more faithful representation of the original function shapes while performing curve registration.
\end{abstract}

\section{Introduction}
Several optimization problems involve a twice differentiable function \( f \) being optimized over a feasible set that is a product space of simplices and Euclidean components:
\begingroup\renewcommand{\theHequation}{Psimplex}%
\begin{equation}
\tag{$P_{\text{simplex}}$}
\begin{split}
&\min_{v \in \mathcal{F}_{\mathrm{simplex}}} f(v), \\
&\mathcal{F}_{\mathrm{simplex}} = \Delta^{n_1} \times \cdots \times \Delta^{n_{L}} \times 
\mathbb{R}^{n_{L+1}} \times \cdots \times \mathbb{R}^{n_{L+Q}}.
\end{split}
\label{eq:formulation_simplex}
\end{equation}\endgroup

Here, $\Delta^n = \{ x \in \mathbb{R}^n \mid x \geq 0, \textbf{1}^T x=1\}$ denotes the standard $(n-1)$-simplex embedded in $\mathbb{R}^{n}$. Such problems often naturally arise in diverse domains, such as (i) \emph{simplex-constrained matrix factorization}, where the elements of the simplices may for example represent unknown mixture proportions of pure minerals to be estimated from hyperspectral images \cite{abdolali2021simplex,kizel2017stepwise,ma2013signal,wu2017stochastic}; (ii) \emph{probabilistic tensor decomposition}, where simplices are used to learn probability distributions \cite{kargas2018tensors}, perform feature selection \cite{amiridi2021information}, or minimize the smallest element of a tensor \cite{sidiropoulos2023minimizing}; and (iii) \emph{shape analysis}, where the Square Root Velocity Function (SRVF) representation leverages simplicial parameterizations to temporally align functional data \cite{srivastava2016functional,tucker2013generative}. 

A common technique for solving optimization problems with such constraints involves strategies like Projected Gradient Descent (PGD) which involve moving along the direction of the gradient followed by projection on $\mathcal{F}_{\mathrm{simplex}}$ \cite{kizel2017stepwise}. As an alternative for solving \eqref{eq:formulation_simplex}, several papers have focused on solving by Riemannian optimization algorithms a smooth reparameterized version of the problem:
\begingroup\renewcommand{\theHequation}{Psmooth}%
\begin{equation}
\tag{$P_{\text{smooth}}$}
\label{eq:smooth}
\min_{\theta \in \mathcal{F}_{\mathrm{smooth}}} f(\varphi(\theta)),
\end{equation}\endgroup
where $\mathcal{F}_{\mathrm{smooth}}$ is a manifold and $\varphi: \mathcal{F}_{\mathrm{smooth}} \to \mathcal{F}_{\mathrm{simplex}}$ is surjective.

For instance, the Hadamard reparameterization $\varphi:\mathbb{S}^{n-1}\to\Delta^{n}$, defined by $\varphi(\theta)=\theta\odot\theta$, where $\mathbb{S}^{n-1}:=\left\{\theta\in\mathbb{R}^{n}:\|\theta\|_{2}=1\right\}$ denotes the unit sphere in $\mathbb{R}^{n}$, has been used in linear programming \cite{faybusovich1991hamiltonian}, sparse recovery \cite{vaskevicius2019implicit}, and high-dimensional linear regression \cite{zhao2019implicit}. Products of Hadamard maps $\bigl(\varphi: \mathbb{S}^{n-1} \times \dots \times \mathbb{S}^{n-1} \to \Delta^{n} \times \dots \times \Delta^{n},\ \varphi(\theta) = \bigl(\theta_1 \odot \theta_1, \dots, \theta_L \odot \theta_L\bigr)\bigr)$ have been used to solve Simplex Constrained Matrix Factorization in \cite{esposito2025riemannian} and \cite{guo2021sparse}. Other works also add Euclidean components for use in  shape registration \cite{tucker2013generative} $\bigl(\varphi: \mathbb{S}^{n-1} \times \dots \times \mathbb{S}^{n-1} \times \mathbb{R}^n \to \Delta ^{n} \times \dots \times \Delta^n \times \mathbb{R}^n,\ \varphi(\theta) = \bigl(\theta_1 \odot \theta_1,\dots,\theta_L\odot \theta_L,\theta_{L+1}\bigr)\bigr)$. Non-Hadamard parameterizations $\bigl(\varphi: \mathbb{R}^n \times \dots \times \mathbb{R}^n \to \Delta^n \times \dots \times \Delta^n,\ \varphi(\theta) = \bigl(\frac{e^{\theta_1}}{\mathbf{1}^T e^{\theta_1}}, \dots, \frac{e^{\theta_L}}{\mathbf{1}^T e^{\theta_L}}\bigr)\bigr)$ have also been used in probabilistic tensor factorizations \cite{sidiropoulos2023minimizing}. These algorithms can solve \eqref{eq:smooth} by moving along the manifold $\mathcal{F}_{\mathrm{smooth}}$ without requiring projection back to $\mathcal{F}_{\mathrm{simplex}}$.

A recent body of work has studied the critical points of \eqref{eq:formulation_simplex} and \eqref{eq:smooth} under the Hadamard map, $\varphi: \mathbb{S}^{n-1} \to \Delta^{n},\ \varphi(\theta) = \theta \odot \theta$, showing that second order KKT points for \eqref{eq:smooth} imply first order \cite{levin2022effect} and (weak) second-order criticality \cite{li2021simplex} for \eqref{eq:formulation_simplex}, respectively. However, several of the applications we described above involve more general feasible sets and transformations, and the existing results do not directly extend to these settings. Furthermore, while the Hadamard map has recieved a lot of attention, it is not known if there are other more general classes of transformations which preserve KKT points.

\paragraph{\textbf{Contributions:}}
\begin{itemize}
    \item We generalize the results of Li et al.~\cite{li2021simplex} from Hadamard parameterization to a more general class of smooth transformations $\psi(\theta)$, with the following properties: 
\begin{equation}
    \begin{split}
        \psi: \mathbb{R} \rightarrow \mathbb{R}_{\geq 0}, \\
        \psi(0) = 0, \quad \psi'' > 0.
    \end{split}
    \label{eq:psi_conditions}
\end{equation}
    
    Apart from the Hadamard map, this includes  polynomials with even degree, positive coefficients, and no constant term, e.g., $\psi(\theta) = a \theta^4 + b \theta^2, \ a,b > 0$, and other functions such as $\psi(\theta) = \cosh\theta - 1$. 
    
    \item We also generalize their results from $\mathcal{F}_{\text{simplex}}=\Delta^n$ to $\mathcal{F}_{\mathrm{simplex}} = \Delta^{n_1} \times \cdots \times \Delta^{n_L} \times \mathbb{R}^{n_{L+1}} \times \cdots \times \mathbb{R}^{n_{L+Q}}$, allowing this result to be applied to problems like matrix factorization \cite{esposito2025riemannian}, tensor factorization \cite{sidiropoulos2023minimizing}, and functional data registration \cite{tucker2013generative}. This makes it possible to tackle the problems by Riemannian optimization algorithms~\cite{AMS2008,boumal2023intromanifolds}.

    \item We distinguish a weak and a strong version of the second-order KKT conditions for~\eqref{eq:formulation_simplex}. We show that in Section~\ref{sec:main_results} that the weak---but not the strong---second-order KKT points of~\eqref{eq:formulation_simplex} correspond to the (weak, equivalently strong) second-order KKT points of the parameterized formulation.

    \item We provide numerical results on two problems in particular: a) simplex constrained tensor factorization and b) shape registration. We introduce a Riemannian Gradient Descent (RGD) algorithm which outperforms PGD for simplex constrained tensor factorization problems and achieves smoother and more accurate registration for the shape registration problem.
\end{itemize}

\section{Reparameterization of \eqref{eq:formulation_simplex}}
We reparameterize \eqref{eq:formulation_simplex} under transformations of the form $v = \varphi(\theta) = [\psi(\theta_1),\dots,\psi(\theta_L),\theta_{L+1},\dots, \theta_{L+Q}]$ with $\psi$ as in~\eqref{eq:psi_conditions}. This leads to a reparameterized problem:
\begingroup\renewcommand{\theHequation}{Prodsmooth}%
\begin{equation}
\tag{$P_{\times\text{smooth}}$}
\begin{split}
&\min_{\theta \in \mathcal{F}_{\text{smooth}}} f(\varphi(\theta)), \\
&\mathcal{F}_{\text{smooth}} = \{ \theta \in \mathbb{R}^{n_1} \times \dots \times \mathbb{R}^{n_{L+Q}}
\mid \mathbf{1}^T \psi(\theta_l) = 1, \quad l = 1, \dots, L\}.
\end{split}
\label{eq:formulation_smooth}
\end{equation}\endgroup
Here, function $\psi$ applied to a vector is meant elementwise: $(\psi(\theta_l))^i = \psi(\theta_l^i)$. As for $\psi^{\prime}$ and $\psi^{\prime\prime}$, they are the elementwise first and second-order derivative, respectively. Fig.~\ref{fig:reparam} shows an example of this type of reparameterization.
\begin{figure}[!t]
    \centering
    \includegraphics[width=0.99\linewidth]{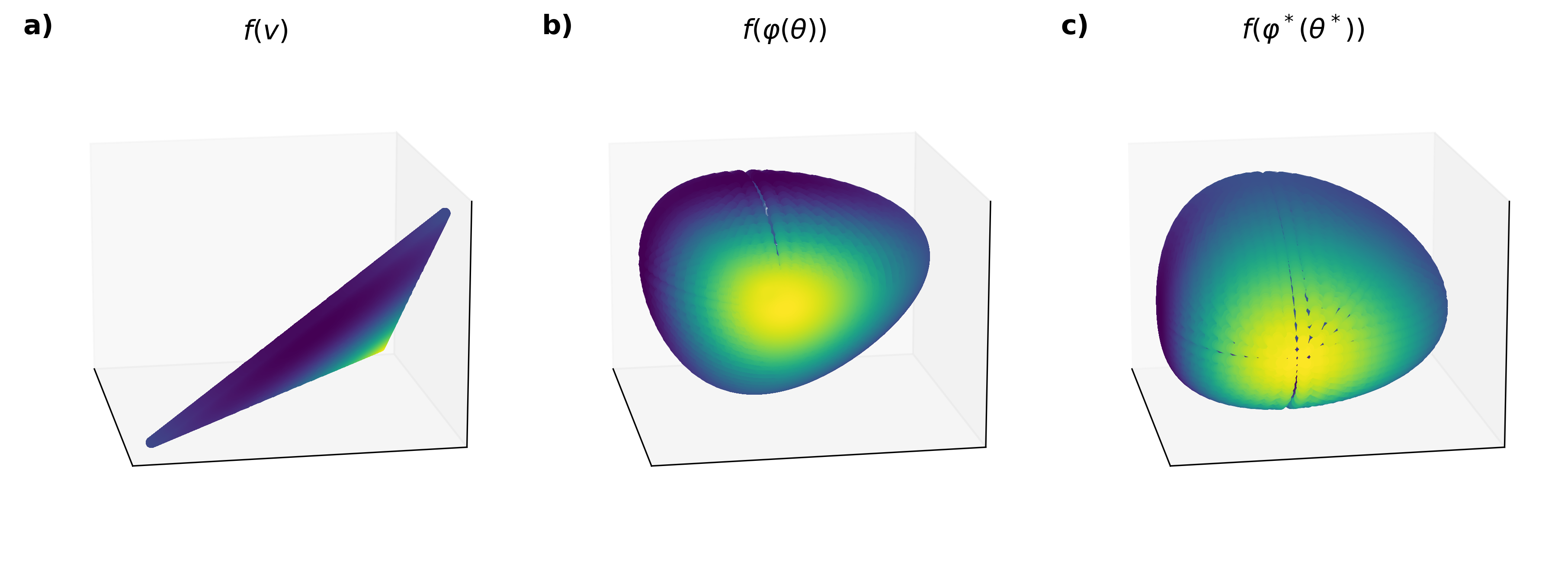}
    \caption{a) Example of a function $f(v) = v^T A v$ on $\mathcal{F}_{\mathrm{simplex}}$. b) Reparameterizing $f(v)$  on a smooth manifold $\mathcal{F}_{\mathrm{smooth}}$ as $f(\varphi(\theta))$, and c) its Legendre transform $f(\varphi^*(\theta^*))$. Darker colors correspond to lower values of $f$. At the lower right corner of the simplex, all the equations in the KKT conditions~\eqref{eq:kkt_simplex} can be satisfied except $\beta_l \geq 0$, hence the lower right corner is not a first-order KKT point; its counterpart in $\mathcal{F}_{\mathrm{smooth}}$ is the point with the lightest color which, being a maximizer subject to a submanifold constraint, is a KKT point.}
    \label{fig:reparam}
\end{figure}
Two important properties of $\psi$ immediately stand out:
\begin{remark}
\label{rem:psi_derivative}
$\psi'(\theta) = 0$ if and only if $\theta = 0$. Indeed, $\psi$ is smooth and $0$ is a minimizer, hence $\psi'(0) = 0$; and if $\theta \neq 0$, then the mean value theorem applied to $\psi'$ yields $\psi'(\theta) = \psi''(\tilde\theta)\theta$, which is nonzero since $\psi''>0$.
\end{remark}
\begin{remark}
\label{rem:zero_root}
$\psi(\theta) = 0$ if and only if $\theta = 0$. Indeed, the condition is sufficient by definition of $\psi$; for the converse, observe that  Taylor's theorem with the remainder in Lagrange form yields $\psi(\theta) = \frac12 \psi''(\tilde\theta) \theta^2 > 0$ for all $\theta \neq 0$.
\end{remark}
\begin{remark}
The range of $\psi$ is $\mathbb{R}_{\geq 0}$. Indeed, since $\psi''>0$,
the function $\psi$ is convex and nonconstant, and therefore it cannot be
bounded from above, as every convex function on $\mathbb{R}$ that is bounded
from above is constant~\cite[\S~1.5, Exercise~2]{niculescu2006convex}.
Since $\psi(0)=0$, $\psi\geq 0$, and $\psi$ is continuous, the intermediate
value theorem yields $\operatorname{Im}(\psi)=\mathbb{R}_{\geq 0}$.
\end{remark}
\begin{proposition}
The feasible set \( \mathcal{F}_{\mathrm{smooth}} \) is a Riemannian manifold.
\end{proposition}
\begin{proof}
Let us show that $\mathcal{F}_{\mathrm{smooth}}$ is a Cartesian product of embedded submanifolds. In view of~\cite[Proposition 3.3.3]{AMS2008}, it suffices to show that the differential of $c: \mathbb{R}^{n_l} \to \mathbb{R}: \theta_l \mapsto \mathbf{1}^T \psi(\theta_l)$ is surjective, namely nonzero, at all $\theta_l$ that satisfy the simplex constraint $c(\theta_l)=1$. The Jacobian of $c$ is the row vector $(\psi'(\theta_l))^T$. 
In view of the above remarks, the simplex constraint $c(\theta_l) = 1$ implies that at least one component of vector $\theta_l$ is nonzero, which ensures that $(\psi'(\theta_l))^T$ is nonzero.
\qed \end{proof}
We now study relationships between critical points of~\eqref{eq:formulation_simplex} and~\eqref{eq:formulation_smooth}.
\section{KKT Analysis}
Constructing the Lagrangians of~\eqref{eq:formulation_simplex} and~\eqref{eq:formulation_smooth}, we get, 
\begin{gather}
    L_1(v, \lambda, \beta) = f(v) + \sum_{l=1}^L \lambda_l (\mathbf{1}^T v_l - 1) - \beta_l^T v_l, \label{eq:lagrangian_simplex} \\[6pt]
    L_2(\theta, \lambda^{\mathrm{s}}) = f(\varphi(\theta)) + \sum_{l=1}^L \lambda^{\mathrm{s}} (\mathbf{1}^T \psi(\theta_l) - 1). \label{eq:lagrangian_sphere}
\end{gather}
The first-order KKT conditions of~\eqref{eq:formulation_simplex} take the form
\begin{subequations}\label{eq:kkt_simplex}
\begin{align}
        \nabla_{v_l} f(v) + \lambda_l \mathbf{1} - \beta_l = 0, & \quad l = 1,\dots,L, \label{eq:kkt_simplex_zero_gradient}\\
        \beta_l \odot v_l = 0, & \quad l = 1,\dots,L,  \label{eq:kkt_simplex_hadamard_zero}\\
        v_l \geq 0, \ \beta_l \geq 0, & \quad l = 1,\dots,L, \label{eq:kkt_simplex_nonnegative}\\
        \mathbf{1}^T v_l = 1, & \quad l = 1,\dots,L, \label{eq:kkt_point_simplex}\\
        \nabla_{v_q} f(v) = 0, & \quad q = L+1,\dots,L+Q,\label{eq:kkt_simplex_zero_euclidean_gradient}
\end{align}
\end{subequations}
and those of~\eqref{eq:formulation_smooth} read 
\begin{subequations}\label{eq:kkt_sphere}
    \begin{align}
        \psi^{\prime}(\theta_l) \odot \nabla_{v_l} f(v) + \lambda^{\mathrm{s}} \psi^{\prime}(\theta_l)  = 0, & \quad l = 1,\dots,L, \label{eq:kkt_smooth_zero_gradient}\\
        \mathbf{1}^T \psi(\theta_l) = 1, & \quad l = 1,\dots,L,  \label{eq:kkt_point_smooth}\\
        \nabla_{\theta_q} f(\varphi(\theta)) = 0, & \quad q = L+1,\dots,L+Q.\label{eq:kkt_smooth_zero_euclidean_gradient}
    \end{align}  
\end{subequations}
They reduce to those in Li et al.~\cite[\S 4]{li2021simplex} when $L=1$, $Q=0$, and $\psi(\theta) = \theta \odot \theta$.

Let $v^\star$ and $\theta^\star$ be feasible points. Assume that $(v^\star,\lambda^\star,\beta^\star)$ and
$(\theta^\star,\lambda^{s,\star})$ satisfy the first-order KKT conditions~\eqref{eq:kkt_simplex} and~\eqref{eq:kkt_sphere}, respectively. The weak second-order KKT conditions for~\eqref{eq:formulation_simplex} are given by
\begin{subequations}
    \label{eq:second_order_kkt_simplex}
    \begin{align}
        \Delta v^T \nabla^2_v L_1 \left(v^\star,\lambda^\star,\beta^\star\right) \Delta v \geq 0, & \quad \Delta v = \begin{bmatrix} \hdots &,\Delta {v_l}, & \hdots & ,\Delta {v_q}, & \hdots \end{bmatrix}, \label{eq:kkt_positive_definite_simplex}\\
        \mathbf{1}^T \Delta v_l = 0, &\quad l = 1,\dots,L,\label{eq:kkt_tangent_vector_simplex}\\
        \Delta v_l^i = 0 \quad \text{if } (v_l^\star)^i = 0,  &\quad l = 1,\dots,L,\label{eq:weak_kkt}
        \end{align}
\end{subequations}
and the (weak, equivalently strong) second-order KKT conditions for~\eqref{eq:formulation_smooth} are given by
\begin{subequations}
    \label{eq:second_order_kkt_sphere}
    \begin{align}
        \Delta \theta^T \nabla^2_{\theta} L_2 \left(\theta^\star,\lambda^{s,\star}\right) \Delta \theta \geq 0, &\quad\Delta \theta = \begin{bmatrix} \hdots & ,\Delta {\theta_l}, & \hdots & ,\Delta {\theta_q}, & \hdots \end{bmatrix},\label{eq:kkt_positive_definite_smooth}\\
         \quad \mathbf{1}^T(\psi^{\prime}(\theta_l^\star) \odot \Delta \theta_l) = 0, &\quad l = 1,\dots,L.\label{eq:kkt_tangent_vector_smooth}
    \end{align}
\end{subequations}
\begin{remark}
Equations~\eqref{eq:kkt_sphere} express that the Riemannian gradient is zero for the Riemannian optimization problem~\eqref{eq:formulation_smooth}. Likewise,~\eqref{eq:second_order_kkt_sphere} express that the Riemannian Hessian is positive semidefinite.    
\end{remark}
\begin{remark}
    Condition~\eqref{eq:second_order_kkt_simplex} represents a weak form of the second-order necessary conditions. A \emph{weak second-order KKT point} is a point $v^\star$ for which there exist $\lambda^\star$ and $\beta^\star$ such that~\eqref{eq:kkt_simplex} and~\eqref{eq:second_order_kkt_simplex} are satisfied. Note that this is distinct from the classical second-order necessary conditions for optimality. While Bertsekas et al.~\cite{bertsekas1997nonlinear} and Gill et al.~\cite{gill2019practical} present similar weak conditions, Bonnans et al.~\cite{bonnans2006numerical}, Nocedal et al.~\cite{nocedal2006numerical}, and Arag\'{o}n et al.~\cite{aragon2019nonlinear} present the strong conditions. For the strong version, we replace
    \begin{align*}
        \Delta v_l^i = 0 \quad \text{if } (v_l^\star)^i = 0
    \end{align*}
with
    \begin{align*}
        \Delta v_l^i = 0 \quad &\text{if } (v_l^\star)^i = 0 \text{ and } (\beta_l^\star)^i > 0, \\
        \Delta v_l^i \geq 0 \quad &\text{if } (v_l^\star)^i = 0 \text{ and } (\beta_l^\star)^i = 0.
    \end{align*}
For \eqref{eq:formulation_smooth}, the weak and strong conditions coincide because there are no inequality constraints. The weak/strong terminology follows the one in~\cite[Definition~2.1]{GuoLinYe2013}.
\end{remark}

\section{Main Results}\label{sec:main_results}
The Hessian from the second-order KKT of~\eqref{eq:formulation_simplex} is given by:
\begin{equation}
    \nabla_v^2 L_1(v,\lambda,\beta) = \nabla_v^2 f(v).
\end{equation}
Let \( \theta_a \) and \( \theta_b \) be two blocks of variables with dimensions $n_a$ and $n_b$ respectively. Let \( g(\theta) \) be a differentiable function. Define matrix $\nabla^2_{\theta_a \theta_b} g(\theta) \in \mathbb{R}^{n_a \times n_b}$ by:
\[
\left( \nabla^2_{\theta_a \theta_b} g(\theta) \right)_{ij} 
= \frac{\partial^2 g(\theta)}{\partial \theta_a^i \partial \theta_b^j}.
\] 
Next, consider the Hessian of~\eqref{eq:formulation_smooth}. Here, similar to Li et al.~\cite[Equation~13]{li2021simplex}, the Hessian is the sum of a diagonal term and a generally non-diagonal term involving the Hessian of $f$. However, our formulation includes multiple simplex and Euclidean blocks, leading to mixed second-derivative blocks, so the results of Li et al.~\cite{li2021simplex} do not apply directly.
For $l,m \in \{1, \dots, L\}$ with $l\neq m$ and $q \in \{L+1, \dots, L+Q\}$, we obtain
\begin{equation}
    \begin{split}
    \nabla^2_{\theta_q \theta_q} L_2 &= \nabla^2_{v_q v_q} f, \\
    \nabla^2_{\theta_q \theta_l} L_2 &= \nabla^2_{v_q v_l} f \, \mathrm{diag}(\psi^{\prime}(\theta_l)), \\
    \nabla^2_{\theta_l \theta_q} L_2 &= \mathrm{diag}(\psi^{\prime}(\theta_l)) \nabla^2_{v_l v_q} f, \\
    \nabla^2_{\theta_l \theta_l} L_2 &= \mathrm{diag}(\psi^{\prime\prime}(\theta_l) \odot \nabla_{v_l} f + \lambda_l^{\mathrm{s}} \psi^{\prime\prime}(\theta_l)) + \mathrm{diag}(\psi^{\prime}(\theta_l)) \nabla^2_{v_l v_l} f \, \mathrm{diag}(\psi^{\prime}(\theta_l)), \\
    \nabla^2_{\theta_m \theta_l} L_2 &= \mathrm{diag}(\psi^{\prime}(\theta_m)) \nabla^2_{v_m v_l} f \, \mathrm{diag}(\psi^{\prime}(\theta_l)).
    \label{eq:sphere_hessian_terms}
    \end{split}
\end{equation}
Using~\eqref{eq:sphere_hessian_terms}, we can rewrite $\nabla^2_{\theta} L_2$ as the sum of a diagonal matrix and a change of basis Hessian term
\begin{equation}
\begin{split}
\nabla^2_{\theta} L_2  
\!=\!\!\begin{bmatrix}
\ddots & & \\
& \mathrm{diag}(\psi^{\prime\prime}(\theta_l) \odot \nabla_{v_l} f + \lambda_l^{\mathrm{s}} \psi^{\prime\prime}(\theta_l)) & \\
& \ddots & \\
& 0_{q} & \\
& & \ddots
\end{bmatrix}\!\!\!+\!\mathrm{diag}(\varphi^{\prime}(\theta))\,\nabla^2_{v} f\,\mathrm{diag}(\varphi^{\prime}(\theta)).
\end{split}
\label{eq:hessian_rearrangement}
\end{equation}
Here, $\operatorname{diag}\bigl(\varphi'(\theta)\bigr)
=
\begin{bmatrix}
\ddots &        &        &        &        \\
       & \psi'(\theta_l) &        &        &        \\
       &        & \ddots &        &        \\
       &        &        & I_q    &        \\
       &        &        &        & \ddots
\end{bmatrix}$
is the Jacobian of $\varphi$. 

We now proceed to show the following:
\begin{theorem}
\label{thm:simplex_1_smooth_1}
If $v^\star$ is a first-order KKT point of~\eqref{eq:formulation_simplex} then all $\theta^\star$ satisfying $v^\star = \varphi(\theta^\star)$ are first-order KKT points of~\eqref{eq:formulation_smooth}.
\end{theorem}

\begin{proof}
We extend Li et al.~\cite[Equations 16–17]{li2021simplex} to the reparameterization proposed in \eqref{eq:formulation_smooth}. Let us start with the KKT zero-gradient conditions with respect to each vector lying in $\mathcal{F}_{\mathrm{simplex}}$  \eqref{eq:kkt_simplex_zero_gradient},
\begin{equation}
    \nabla_{v_l} f + \lambda_l \mathbf{1} - \beta_l = 0, \quad l = 1,\dots,L.
\end{equation}
Hadamarding with $\psi^{\prime}(\theta_l)$ and setting $\lambda_l=\lambda_l^{\mathrm{s}}$,
\begin{equation}
    \nabla_{v_l} f \odot \psi^{\prime}(\theta_l) + \lambda_l^{\mathrm{s}} \psi^{\prime}(\theta_l) - \beta_l \odot \psi^{\prime}(\theta_l) = 0, \quad l = 1,\dots,L.
\end{equation}
From the first-order KKT condition on the simplex \eqref{eq:kkt_simplex_hadamard_zero}, we have
\begin{equation}
    \beta_l \odot v_l = 0, \quad l = 1,\dots,L.
\end{equation}
This implies the following two cases:
\begin{itemize}
    \item If \( v_l^i = 0 \), then \( \theta_l^i = 0 \), leading to \( \psi'(\theta_l)^i = 0 \) (see Remarks~\ref{rem:psi_derivative}, \ref{rem:zero_root}), and consequently \( \beta_l^i \psi'(\theta_l)^i = 0 \).
    \item If \( v_l^{i} \neq 0 \), then \( \beta_l^i = 0 \), which also ensures \( \beta_l^i \psi'(\theta_l)^i = 0 \).
\end{itemize}
Thus, we conclude that
\begin{equation}
    \beta_l \odot v_l = 0 \implies \beta_l \odot \psi'(\theta_l) = 0, \quad l = 1,\dots,L.
\end{equation}
Plugging this into the equation above satisfies the zero-gradient conditions of~\eqref{eq:kkt_smooth_zero_gradient}:
\begin{equation}
    \nabla_{v_l} f \odot \psi^{\prime}(\theta_l) + \lambda_l^{\mathrm{s}} \psi^{\prime}(\theta_l) = 0, \quad l = 1,\dots,L.
\end{equation}
Next, the simplex constraint~\eqref{eq:kkt_point_smooth} yields
\begin{equation}
    \mathbf{1}^T v_l = 1 \implies \mathbf{1}^T \psi(\theta_l) = 1, \quad l = 1,\dots,L.
\end{equation}
Finally, the Euclidean gradient conditions~\eqref{eq:kkt_smooth_zero_euclidean_gradient} can be rewritten as
\begin{equation}
    \nabla_{v_q} L_1 = 0 \implies \nabla_{v_q} f = 0 \implies \nabla_{\theta_q} f = 0 \implies \nabla_{\theta_q} L_2 = 0.
\end{equation}
Thus, a first-order KKT point on the product simplex implies a first-order KKT point on the smooth manifold.
\qed
\end{proof}
\begin{theorem}
\label{thm:simplex_2_smooth_2}
If $v^\star$ is a weak second-order KKT point of~\eqref{eq:formulation_simplex} then all $\theta^\star$ satisfying $v^\star = \varphi(\theta^\star)$ are second-order KKT points of~\eqref{eq:formulation_smooth}.
\end{theorem}
\begin{proof}
We now use our Hessian formula in~\eqref{eq:hessian_rearrangement} to extend Li et al.~\cite[Equations 18-23]{li2021simplex} to show that second-order KKT points on $\mathcal{F}_{\mathrm{simplex}}$ implies second-order KKT points on $\mathcal{F}_{\mathrm{smooth}}$. We start with an arbitrary vector in the tangent space of $\mathcal{F}_{\mathrm{smooth}}$:
\begin{equation}
    \Delta \theta = \begin{bmatrix}
        \hdots \Delta \theta_q & \hdots & \Delta \theta_l & \hdots 
    \end{bmatrix},
    \quad \Delta\theta_l \ \text{s.t.}\ \mathbf{1}^T(\psi'(\theta_l) \odot \Delta\theta_l) = 0, \quad l = 1,\dots,L.
\end{equation}
Since $\psi^{\prime}(\theta_l) \odot \Delta \theta_l = \mathrm{diag}(\psi^{\prime}(\theta_l)) \Delta \theta_l = \Delta v_l$,
the corresponding vector $\Delta v_l$ after applying the Jacobian to $\Delta \theta_l$ satisfies $1^T \Delta v_l = 0$, $l = 1,\dots,L$.  Furthermore, if $v_l^i=0$, then $\theta_l^i=0$ and $\psi'(\theta_l)^i=0$ (Remark \ref{rem:psi_derivative}, ~\ref{rem:zero_root}) and thus $\Delta v_l^i=0$. $\Delta v_l$ is thus a valid tangent vector on the corresponding simplex.
\begin{multline}
        \Delta \theta^T \nabla^2_\theta L_2 \Delta \theta = \\
         \begin{bmatrix}
         \hdots \Delta \theta_q & \hdots & \Delta \theta_l & \hdots 
         \end{bmatrix}
        \begin{bmatrix}
        \ddots & & \\
         & \mathbf{0}_{q} & \\
         & \ddots & \\
        &  \mathrm{diag}(\psi^{\prime\prime}(\theta_l) \odot \nabla_{v_l} f + \lambda_l^{\mathrm{s}} \psi^{\prime\prime}(\theta_l)) & \\
        & & \ddots
    \end{bmatrix}
        \begin{bmatrix}
        \vdots \\ \Delta \theta_q \\ \vdots \\ \Delta \theta_l \\ \vdots 
    \end{bmatrix} + \\
        \hfill
        \Delta \theta^T \mathrm{diag}(\varphi^{\prime}(\theta)) \nabla^2_{v} f \mathrm{diag}(\varphi^{\prime}(\theta)) \Delta \theta\\
        \\
         = \begin{bmatrix}
            \Delta \theta_1 & \dots &\Delta \theta_L
        \end{bmatrix}
        \begin{bmatrix}
        \ddots & & \\
        &  \mathrm{diag}(\psi^{\prime\prime}(\theta_l) \odot (\nabla_{v_l} f + \lambda_l^{\mathrm{s}} \mathbf{1})) & \\
        & & \ddots
    \end{bmatrix}
        \begin{bmatrix}
            \Delta \theta_1 \\ \vdots \\ \Delta \theta_L
        \end{bmatrix} + \\
        \hfill
        \Delta v^T \nabla^2_{v} f \Delta v \\
        \\
        = \begin{bmatrix}
            \Delta \theta_1 & \dots &\Delta \theta_L
        \end{bmatrix}
        \begin{bmatrix}
            \ddots &  & \\
             & \mathrm{diag}(\psi^{\prime\prime}(\theta_l) \odot \beta_l) &  \\
             &  & \ddots
        \end{bmatrix}
        \begin{bmatrix}
            \Delta \theta_1 \\ \vdots \\ \Delta \theta_L
        \end{bmatrix} + \Delta v^T \nabla^2 f \Delta v \\
        = \sum_{l=1}^L \sum_{i=1}^{n_l} \psi^{\prime\prime}(\theta_l)^i \beta_l^i (\Delta \theta_l^i)^2 + \Delta v^T \nabla^2 f \Delta v.
        \label{eq:kkt_hessian_rhs}
\end{multline}
Note that $\beta_l^i \geq 0$ \eqref{eq:kkt_simplex_nonnegative}, $\psi^{\prime\prime}(\theta_l)^i > 0$ \eqref{eq:psi_conditions} and $\Delta v^T \nabla^2 f \Delta v$ is nonnegative \eqref{eq:kkt_positive_definite_simplex}, implying that \eqref{eq:kkt_hessian_rhs} is nonnegative.
\qed \end{proof}

It is interesting to note that Theorem \ref{thm:simplex_1_smooth_1} does not imply bidirectional criticality. We need an additional assumption of second-order criticality of \eqref{eq:formulation_smooth} in order to prove first-order criticality of~\eqref{eq:formulation_simplex}. Recall that the weak and strong second-order KKT points coincide for~\eqref{eq:formulation_smooth}.
\begin{theorem}
\label{thm:smooth_2_simplex_1}
\leavevmode
If $\theta^\star$ is a second-order KKT point of~\eqref{eq:formulation_smooth} then $v^\star = \varphi(\theta^\star)$ is a first order KKT point of~\eqref{eq:formulation_simplex}.
\end{theorem}
\begin{proof}
    Similar to Li et al.~\cite[Equations 24-31]{li2021simplex}, we show that second-order KKT on smooth manifold implies first-order KKT point on $\mathcal{F}_{\mathrm{simplex}}$. From KKT conditions on the Smooth manifold, we have:
\begin{equation}
    \begin{split}
        \psi^{\prime}(\theta_l) \odot \nabla_{v_l} f + \lambda_l^{\mathrm{s}} \psi^{\prime}(\theta_l) = 0, &\quad l = 1,\dots,L.
    \end{split}
\end{equation}

Since the corresponding point on $\mathcal{F}_{\mathrm{simplex}}$  is not a stationary point, the gradient conditions on the Simplex are:
\begin{equation}
    \begin{split}
        \nabla_{v_l} f + \lambda_l \mathbf{1} - \beta_l = g_l, &\quad l = 1,\dots,L.
    \end{split}
\end{equation}
Absorbing $g_l$ in $\beta_l$, these conditions can be rewritten as:
\begin{equation}
    \begin{split}
        \nabla_{v_l} f + \lambda_l \mathbf{1} - \hat{\beta_l} = 0, &\quad l = 1,\dots,L.
    \end{split}
    \label{eq:beta_absorb}
\end{equation}
Hadamarding equation~\eqref{eq:beta_absorb} with $\psi^{\prime}(\theta_l)$ and setting $\lambda_l = \lambda_l^{\mathrm{s}}$
\begin{equation}
    \begin{split}
        \psi^{\prime}(\theta_l) \odot \nabla_{v_l} f + \lambda_l^{\mathrm{s}} \psi^{\prime}(\theta_l) - \psi^{\prime}(\theta_l) \odot \hat{\beta_l} = 0 \\\implies \psi^{\prime}(\theta_l) \odot \hat{\beta_l} = 0, &\quad l = 1,\dots,L.
    \end{split}
    \label{eq:kkt_hadmard_beta_equal_zero}
\end{equation}
Furthermore because our choice of parameterization, $\psi$ is 0 only at 0 (\ref{rem:zero_root}).
\begin{equation}
    \begin{split}
        \psi^{\prime}(\theta_l) \odot \hat{\beta_l} = 0 \implies \theta_l \odot \hat{\beta_l} = 0 \implies v_l \odot \hat{\beta_l} = 0, &\quad l = 1,\dots,L.
    \end{split}
\end{equation}
Note that because of the nature of the square mapping, the positivity and summing up to 1, KKT conditions of~\eqref{eq:kkt_simplex} are already satisfied.
As long as we can show that $\hat{\beta_l}\geq 0$, $v$ is a corresponding first-order KKT for $\theta$.
We see that because $\theta_l\odot\hat{\beta_l}=0$, $\hat{\beta_l^{i}}=0$ when $\theta_l^i \neq 0$. This implies that we just need to show $\hat{\beta_l^{i}} \geq 0$ when $\theta_l^i= 0$.

The second-order KKT criticality on smooth manifold implies that 
$$\Delta \theta \nabla^2 L_2 \Delta \theta \geq 0, \quad \forall\, \Delta\theta_l \ \text{s.t.}\ 
\mathbf{1}^T(\psi'(\theta_l) \odot \Delta\theta_l) = 0, \quad l = 1,\dots,L.$$
Consider tangent vectors of the type $ \Delta \theta = \begin{bmatrix}
        0 &
        \hdots &
        \delta_l &
        \hdots &
        &
        0 
    \end{bmatrix}$. Here $\delta_l^i$ is set to 1 only if $\theta_l^i=0$. All other components of $\delta_l$ including Euclidean tangent vectors are set to 0. Using this special tangent vector in the equation for the Product Hessian, we get
\begin{equation}
    \begin{split}
        \Delta \theta \nabla^2 L_2 \Delta \theta = \begin{bmatrix}
            0 & \dots & \delta_l^i & \dots & 0
        \end{bmatrix}
        \begin{bmatrix}
        \ddots & & \\
        &  \mathrm{diag}(\psi^{\prime\prime}(\theta_l) \odot (\nabla_{v_l} f + \lambda_l^{\mathrm{s}} 1)) & \\
        & & \ddots
    \end{bmatrix}
        \begin{bmatrix}
            0 \\ \vdots \\ \delta_l^i \\ \vdots \\ 0
        \end{bmatrix} \\ + (\Delta \theta \odot \varphi^{\prime}(\theta))^T \nabla_v^2 f (\Delta \theta \odot \varphi^{\prime}(\theta)).
    \end{split}
\end{equation}
Noting that $\mathrm{diag}(\psi'(\theta_l)) \delta_l = 0$ in view of the choice of $\delta_l$, and recalling~\eqref{eq:kkt_simplex}, we get
\begin{align}
\Delta\theta^T \nabla^2L_2 \Delta\theta
&= \delta_l^T \mathrm{diag}\left( \psi''(\theta_l) \odot \nabla_{v_l}f + \lambda^{\mathrm{s}}_l \psi''(\theta_l) \right) \delta_l \nonumber
\\ &= \delta_l^T \mathrm{diag}\left( \psi''(\theta_l) \odot (\nabla_{v_l}f + \lambda^{\mathrm{s}}_l 1) \right) \delta_l \nonumber
\\ &= \delta_l^T \mathrm{diag}\left( \psi''(\theta_l) \odot \hat \beta_l \right) \delta_l \nonumber
\\ &= \psi''(\theta_l)^i \hat\beta_l^i.
\end{align}
Since the LHS is nonnegative by virtue of being a second-order KKT point and $\psi^{\prime\prime}(\theta_l) \geq 0$ because of our choice of convex parameterization, $\hat\beta_l^i \geq 0$, $l = 1,\dots,L$.
Also,
\begin{equation}
    \nabla_{\theta_q} L_2 = 0 \implies \nabla_{v_q} L_1 = 0.
\end{equation}
Thus second-order criticality on smooth manifold implies first-order criticality on $\mathcal{F}_{\mathrm{simplex}}$ .
\qed \end{proof}

We now strengthen Theorem~\ref{thm:smooth_2_simplex_1} to get the converse of Theorem~\ref{thm:simplex_2_smooth_2}.
\begin{theorem}
\label{thm:smooth_2_simplex_2}
\leavevmode
If $\theta^\star$ is a second-order KKT point of~\eqref{eq:formulation_smooth} then $v^\star = \varphi(\theta^\star)$ is a weak second-order KKT point of~\eqref{eq:formulation_simplex}.
\end{theorem}
\begin{proof}
    Finally, we use our Hessian formula in~\eqref{eq:hessian_rearrangement} along with Li et al.~\cite[Equations 32-39]{li2021simplex} to prove that second-order critical point on smooth manifold implies weak second-order criticality on $\mathcal{F}_{\mathrm{simplex}}$.

\begin{equation}
    \begin{split}
        \Delta \theta_l^i \psi^{\prime}(\theta_l) = \Delta v_l^i \quad \text{if} \quad \theta_l^i \neq 0, \\
        \Delta \theta_l^i = 0 \quad \text{if} \quad \theta_l^i = 0.
    \end{split}
\end{equation}
It is easy to verify that this satisfies the smooth tangent space condition $\langle \Delta \theta_l, \psi^{\prime}(\theta_l) \rangle = 0$, $l = 1,\dots,L$.

Using first-order criticality and the expression for Hessian on smooth manifold we get
\begin{align}
    \Delta \theta^T \nabla^2_v L_2 \Delta \theta
    &=\sum_{l=1}^L \Delta \theta_l^T \mathrm{diag}(\hat{\beta_l} \odot \psi^{\prime\prime}(\theta_l)) \Delta \theta_l \nonumber
    \\&+\Delta \theta^T \mathrm{diag}(\varphi^{\prime}(\theta)) \nabla^2_{v} f \mathrm{diag}(\varphi^{\prime}(\theta)) \Delta \theta \nonumber
    \\&= \sum_{l=1}^L \langle \Delta \theta_l, \hat{\beta_l} \odot \psi^{\prime\prime}(\theta_l) \odot \Delta \theta_l \rangle + \Delta v^T \nabla^2_v f \Delta v \nonumber
    \\&= 0 + \Delta v^T \nabla^2_v f \Delta v.
\end{align}
Here the first term becomes 0 since $\hat{\beta_l^i}$ is 0 if $\theta_l^i$ is non zero, while $\Delta \theta_l^i$ is 0 if $\theta_l^i$ is 0. Since the LHS is nonnegative, the RHS has to be nonnegative, proving our assertion. 
\qed \end{proof}

\begin{remark}
    Theorem~\ref{thm:smooth_2_simplex_2} becomes false if ``weak'' is replaced by ``strong''. To see this, consider the simplex in $\mathbb{R}^3$ and the functions $f(v) = -(v^1 - 1)^2$ and $\varphi(\theta) = \theta \odot \theta$. The point $\theta = (1, 0, 0)$ is both a weak and strong second-order KKT point, but $v = (1, 0, 0)$ is not a strong second-order KKT point. However, in view of Theorem~\ref{thm:smooth_2_simplex_2}, $v$ is a \emph{weak} second-order KKT point. In fact, regardless of the objective function, every vertex of the simplex is a weak second-order KKT point because, under condition~\eqref{eq:second_order_kkt_simplex}, the constraints on $\Delta v$ imply $\Delta v = 0$, and thus $\Delta v^T \nabla^2 L_1 \Delta v = 0 \geq 0$ is satisfied.
\end{remark}

\begin{proposition}
\label{rem:legendre}
Let $\psi(\theta)$ satisfy~\eqref{eq:psi_conditions}, and $\operatorname{Im}(\psi')=\mathbb{R}$, then its Legendre transform $\psi^*(\theta^*) = \sup_{\theta \in \mathbb{R}} \bigl\{ \theta^* \theta - \psi(\theta) \bigr\}$ defines an alternative reparameterization with the same second-order KKT--preserving properties.
\end{proposition}
\begin{proof}
Differentiating $\theta^* \theta - \psi(\theta)$ with respect to $\theta$ and setting equal to zero we get $\theta^* = \psi'(\theta)$. Plugging this into the definition of  Legendre transform we get
\begin{equation}
\psi^*\!\left(\psi'(\theta)\right) = \psi'(\theta)\,\theta - \psi(\theta).
\end{equation}
Setting $\theta=0$ and using $\psi(0)=0$ yields $\psi^*(\psi'(0))=0$ and so $\psi^*(0)=0$ because $\psi'(0) = 0$ according to remark~\ref{rem:psi_derivative}.  
Differentiating $\psi^*(\theta^*) = \theta^* \theta - \psi(\theta)$ with respect to $\theta^*$ gives $(\psi^*)'(\theta^*) = \theta$ and further differentiation gives
\begin{equation}
(\psi^*)''(\theta^*) = \frac{d\theta}{d\theta^*}.
\end{equation}
Differentiating $\theta^* = \psi'(\theta)$ with respect to $\theta$ gives
\begin{equation}
\frac{d\theta^*}{d\theta} = \psi''(\theta),
\end{equation}
and hence by inversion
\begin{equation}
\frac{d\theta}{d\theta^*} = \frac{1}{\psi''(\theta)}.
\end{equation}
Therefore,
\begin{equation}
(\psi^*)''(\theta^*) = \frac{1}{\psi''(\theta)} > 0.
\end{equation}
Therefore $\psi^*$ satisfies the same conditions as in~\eqref{eq:psi_conditions}. As a consequence of Theorems \ref{thm:simplex_1_smooth_1}, \ref{thm:simplex_2_smooth_2}, \ref{thm:smooth_2_simplex_1}, and \ref{thm:smooth_2_simplex_2}, $\varphi^*(\theta^*)$  defines a valid reparameterization that preserves second-order KKT properties.

We showed above that $\psi^*(0)=0$. 
Since
$
\psi^*(\theta^*)=\sup_{\theta\in\mathbb{R}}\{\theta^*\theta-\psi(\theta)\},
$ define the Legendre transform in terms of $\theta$ as 
$L(\theta):=\psi^*(\psi'(\theta))
=\theta\psi'(\theta)-\psi(\theta).$ Since $L'(\theta)=\theta\psi''(\theta)$, and $\psi''(\theta)>0$, $L$ attains its unique minimum at $\theta=0$, where $L(0)=0$. Also, 
$\lim_{\theta\to\pm\infty}L(\theta)=\infty.$
Thus,
$\operatorname{Im}(\psi^*)=\operatorname{Im}(L)=[0,\infty),$
so $\psi^*$ is surjective onto $[0,\infty)$, implying that the induced map to the simplex is surjective.
\qed \end{proof}

\section{Applications}
We now proceed to show two important applications of Theorems \ref{thm:simplex_1_smooth_1}, \ref{thm:simplex_2_smooth_2}, \ref{thm:smooth_2_simplex_1}, and \ref{thm:smooth_2_simplex_2}. 

\subsection{Application 1: Simplex Structured Tensor Decomposition for Learning Probability Distributions}
Consider a set of $N$ random variables $\{X_1, X_2, \dots ,X_N\}$ where each $X_n$ takes discrete values in $[1, I_n]$, $n = 1,\dots,N$. Let $\mathcal{P}$ represent the joint PMF tensor of these random variables where $\mathcal{P}(i_1, \dots, i_N) = \mathrm{Pr}\{ X_1=i_1, \dots, X_N=i_N\}$.

Similar to Kargas et al.~\cite{kargas2018tensors}, we assume that $\mathcal{P}$ admits a low-rank non-negative Canonical Polyadic Decomposition (CPD) with $F$ components, which means that there are $N$ factor matrices $A_1, \dots, A_N$ and a loading vector $\lambda$ such that:
\begin{equation}
    \begin{aligned}
        \mathcal{P} &= \sum_{\phi=1}^F \lambda_\phi A_1(:,\phi) \circ A_2(:,\phi) \circ \cdots \circ A_N(:,\phi) \\
          &\triangleq [[\lambda, A_1, \ldots, A_N]] \\
        \mathbf{1}^T \lambda &= 1, \\
        \mathbf{1}^T A_n(:,\phi) &= 1, \quad n = 1,\dots,N, \quad \phi = 1,\dots,F.
    \end{aligned}
    \label{eq:CPD_writeup}
\end{equation}

The problem is to fit model~\eqref{eq:CPD_writeup} to an empirical tensor $\mathcal{Q}$ estimated from data, according to some fit quality measure. This model can then be used to perform probabilistic topic modelling of sequences \cite{sidiropoulos2017tensor,kargas2018tensors} or estimate mutual information for doing feature selection \cite{amiridi2021information}. Some previous works have focused on Least Squares Error \cite{kargas2018tensors} or Kullback-Leibler (KL) Divergence \cite{yeredor2019estimation}, and in our current work, we use the latter.

We stack our factor matrices into a parameter vector $v$:
\begin{equation}
    \begin{aligned}
        v &= \bigl(\lambda , A_1(:,1) , \dots , A_1(:,F) , \dots , A_N(:,1) , \dots , A_N(:,F)\bigr), \\
        v &\in \Delta^{F-1}\times\Delta^{I_1-1}\times\dots\times\Delta^{I_1-1}\times\dots
        \times\Delta^{I_N-1}\times\dots\times\Delta^{I_N-1}.
        \label{eq:vstack}
    \end{aligned}
\end{equation}

The KL-based objective function then reads
\begin{equation}
    f(v) = \sum_{i_N=1}^{I_N} \dots \sum_{i_1=1}^{I_1} \mathcal{Q}(i_1,\dots,i_N) \ln \frac{\mathcal{Q}(i_1,\dots,i_N)}{\mathcal{\hat{P}}(i_1,\dots,i_N ; v)}
    \label{eq:grad_factor}
\end{equation}
where $\mathcal{\hat{P}}(;v)$ indicates the tensor generated by unstacking the parameter $v$ into the factor matrices and use the outer product in \eqref{eq:CPD_writeup} to generate the PMF tensor $\mathcal{P}$.

 Based on these operations, we introduce the RGD Algorithm, as shown in Algorithm~\ref{alg:riemann} and RGD with Barzilai-Borwein step (RGD BB), outlined in Algorithm~\ref{alg:riemannian_bb} for both applications of this paper. Algorithms \ref{alg:riemann} and \ref{alg:riemannian_bb} also require us to define $\textrm{product\_retraction}$ and $\textrm{product\_transport}$. On the sphere $\mathbb{S}^{n-1}$, we define them by projection, respectively onto $\mathbb{S}^{n-1}$ and onto its tangent space:
\begin{align}
    \text{retraction}_{\theta}: T_\theta \mathbb{S}^{n-1} \to \mathbb{S}^{n-1}, \nonumber\\
    \text{retraction}_{\theta}(\Delta \theta ) = \frac{\theta + \Delta \theta}{\lVert \theta + \Delta \theta  \rVert}_2,\\
     \text{product\_retraction} = \text{retraction} \times \dots \times \text{retraction} 
     \label{eq:retraction}
\end{align}
\begin{align}
    \text{transport}_{\theta_1}: T_\theta \mathbb{S}^{n-1} \to T_{\theta_2} \mathbb{S}^{n-1},
    \nonumber \\
    \text{transport}_{\theta_1}(\Delta \theta) = \Delta \theta-\langle \Delta \theta, \theta_2 \rangle \, \theta_2,\\
    \text{product\_transport} = \text{transport} \times \dots \times \text{transport}
    \label{eq:transport}
\end{align}
They are bonafide retraction and vector transport in view of~\cite[Proposition~5]{AbsMal2012} and~\cite[(8.10)]{AMS2008}.

\begin{algorithm}
\caption{RGD with Backtracking Update Step}\label{alg:riemann}
\begin{algorithmic}[1]
\Require initial step size $\eta>0$, initial reparameterized factor vector $\theta$, empirical tensor $\mathcal{Q}$, Armijo constant $c>0$, and backtracking decay factor $\beta\in(0,1)$.
\State $\alpha \gets \eta$ \Comment{Use $\eta$ as initial step size}
\State $G_\mathrm{R} \gets \prod_{\theta} \nabla_{\theta} f(\theta \odot \theta)$ \Comment{Riemannian Gradient via projection in product tangent space}
\State $m \gets 1$
\State $\textrm{ArmijoFlag} \gets \textrm{False}$
\While{$\textrm{ArmijoFlag}=\textrm{False}$ and $m \leq 25$} \label{line:5}
\State $\hat{\theta} = \textrm{product\_retraction}(\theta, -\alpha*G_\mathrm{R})$  \Comment{See, e.g.,~\eqref{eq:retraction}}
\If{$f(\hat{\theta}) - f(\theta) \leq -c*\alpha*\lVert G_\mathrm{R} \rVert^2$}
    \State $\textrm{ArmijoFlag} = \textrm{True}$  
\Else{}
    \State $\alpha \gets \beta \times \alpha$
\EndIf
\State $m \gets m+1$
\EndWhile
\State \Return $\hat{\theta}$
\end{algorithmic}
\end{algorithm}

\begin{algorithm}
\caption{RGD BB with Backtracking Update Step}\label{alg:riemannian_bb}
\begin{algorithmic}[1]
\Require initial step size $\eta>0$, initial reparameterized factor vector $\theta$, empirical tensor $\mathcal{Q}$, Armijo constant $c>0$, backtracking decay factor $\beta\in(0,1)$, Barzilai--Borwein step-size bounds $\alpha_{\min},\alpha_{\max}$ satisfying $0<\alpha_{\min}<\alpha_{\max}<\infty$, and memory parameter $M\in\mathbb{N}$.
\State $\alpha \gets \eta$ \Comment{Use $\eta$ as initial step size}
\State $G_\mathrm{R} \gets \prod_{\theta} \nabla_{\theta} f(\theta \odot \theta)$ \Comment{Riemannian Gradient via projection in product tangent space}
\State $m \gets 1$
\State $\textrm{ArmijoFlag} \gets \textrm{False}$
\While{$\textrm{ArmijoFlag}=\textrm{False}$ and $m \leq 25$}
    \State $\hat{\theta} = \textrm{product\_retraction}(\theta, -\alpha \cdot G_\mathrm{R})$        \Comment{See, e.g.,~\eqref{eq:retraction}}
    \If{$f(\hat{\theta}) \le \max\limits_{i = \max(1,m-M),\dots,m} f(\theta_i) - c * \alpha * \lVert G_\mathrm{R} \rVert^2$}
        \State $\textrm{ArmijoFlag} \gets \textrm{True}$  
    \Else
        \State $\alpha \gets \beta \cdot \alpha$
    \EndIf
    \State $m \gets m+1$
\EndWhile
\State $G_\mathrm{R}^{\text{new}} \gets \prod_{\theta} \nabla_{\theta} f(\hat{\theta} \odot \hat{\theta})$ \Comment{New gradient}
\State $\tilde{G}_\mathrm{R} \gets \textrm{product\_transport}(G_\mathrm{R}, \theta, \hat{\theta})$ \Comment{See, e.g.,~\eqref{eq:transport}}
\State $S \gets -\alpha \cdot \tilde{G}_\mathrm{R}$ 
\State $Y \gets G_\mathrm{R}^{\text{new}} - \tilde{G}_\mathrm{R}$
\State $\textrm{denom} \gets | \langle S, Y \rangle |$
\If{$\textrm{denom} < 10^{-30}$}
    \State $\alpha_{\mathrm{BB}}  \gets \alpha_{\max}$
\Else
    \State $\alpha_{\mathrm{BB}} \gets \frac{\langle S, S \rangle}{\textrm{denom}}$
\EndIf
\State $\alpha \gets \max(\min(\alpha_{\mathrm{BB}}, \alpha_{\max}), \alpha_{\min})$
\State $\theta \gets \hat{\theta}$
\State \Return $\hat{\theta}, \alpha$ \Comment{Return updated parameters and new step size}
\end{algorithmic}
\end{algorithm}

\subsection{Application 2: SRVF Registration for Functional Data Analysis}
\label{subsec:srvf_registration}

In this section, we briefly discuss the SRVF registration framework, referring to~\cite{kumar2025shape} for further details. In contrast to~\cite{kumar2025shape}, where SRVF registration is performed using dynamic programming and used for Shape PCA and clinical motion analysis, our goal here is to show that the registration problem has a product-simplex structure and can be optimized using the smooth reparameterization framework developed in this paper.

As mentioned in~\cite{kumar2025shape}, let $\{ \beta_l \}_{l=1}^L$ be a collection of curves representing some functional observations, each $\beta_l:[0,1]\to\mathbb{R}$. Let $\mathrm{\Gamma}$ denote the set of smooth, strictly increasing maps $\gamma:[0,1]\to[0,1]$ satisfying $\gamma(0)=0$, $\gamma(1)=1$, and $\dot{\gamma}>0$. Elements of $\mathrm{\Gamma}$ are positive diffeomorphisms, also called phases, and are used to model temporal alignments. For a curve $\beta_l$ and a warping function $\gamma_l\in\mathrm{\Gamma}$, the composition $\beta_l\circ\gamma_l$ represents a time-warped version of $\beta_l$.

We first describe the alignment problem in the pairwise setting. Given two curves, $\beta_1$ and $\beta_2$, the goal is to find a time-warping function $\gamma_2$ such that the peaks and valleys of
$\beta_2\circ\gamma_2$
are aligned with those of $\beta_1$. One way to formulate this problem is $\gamma_2^\star\in \arg\min_{\gamma\in\mathrm{\Gamma}} \left\lVert \beta_1 - \beta_2\circ\gamma \right\rVert_2$ where $\left\lVert f \right\rVert_2 =\sqrt{\int_0^1 f(t)^2dt}$
denotes the classical $\mathbb{L}^2$ norm. However, this direct formulation has several mathematical and practical limitations. In particular, this optimization is degenerate because $\| \cdot \|_2$ is not invariant to reparameterization of time and $\|\beta_1 - \beta_2\|_2 \neq \|\beta_1\circ\gamma - \beta_2\circ\gamma\|_2$.

The SRVF framework addresses this limitation by using an elastic Riemannian metric that is invariant to simultaneous time warpings of the input curves. The use of this metric is simplified using a mathematical representation called the square-root velocity function (SRVF). For a curve $\beta_l(t)$, the SRVF is defined as $q_l(t)= \operatorname{sign}(\dot{\beta}_l(t))\sqrt{|\dot{\beta}_l(t)|} = \frac{\dot{\beta}_l(t)}{\sqrt{|\dot{\beta}_l(t)|}}.$ In case, $\dot{\beta}_l(t) = 0$, we set $q_l(t) = 0$.
Under a time warping of $\beta_l$ by  $\gamma_l$, {\it i.e.}, $\beta_l \mapsto \beta_l \circ \gamma_l$, the SRVF transforms as $\gamma_l\cdot q_l=(q_l\circ\gamma_l)\sqrt{\dot{\gamma}_l}.$
This transformation is the basis of the elastic approach to curve alignment. Given two curves $\beta_1$ and $\beta_2$, the pairwise SRVF registration problem is
\begin{equation}
    \gamma_2^\star
    =
    \arg\min_{\gamma\in\mathrm{\Gamma}}
    \left\lVert
    q_1 - (q_2\circ\gamma)\sqrt{\dot{\gamma}}
    \right\rVert_2^2.
    \label{eq:srvf_pairwise_registration}
\end{equation}
The corresponding SRVF distance is invariant to time warpings of the input curves, so it compares shapes rather than the rates at which the curves are traversed.

For multiple curves, the registration problem can be written as
\begin{equation}
    \hat{\mu}_L  = \min_{\mu}
    \left(\sum_{l=1}^L \min_{\gamma_l}
    \left\lVert
    \mu - (q_l\circ\gamma_l)\sqrt{\dot{\gamma}_l}
    \right\rVert^2 \right),
    \label{eq:srvf_multiple_registration}
\end{equation}
where $\hat{\mu}_L$ is the mean SRVF shape and optimal $\gamma_l$ aligns the $l^{th}$ curve to this mean.

We now use this formulation to obtain an optimization problem over products of simplices. Let
$s_t=t/T$, $t=0,\dots,T$,
be a uniform grid on $[0,1]$ with spacing $\Delta t=1/T$. Thus, the grid has $T+1$ points and $T$ intervals. For each warping function, define the discrete warping-rate increment variable
\begin{equation}
    v_l(t)
    =
    \gamma_l(s_{t+1})-\gamma_l(s_t),
    \qquad
    t=0,\dots,T-1.
\end{equation}

Equivalently, when the warping function is represented through its derivative,
\begin{equation}
    v_l(t)
    \approx
    \dot{\gamma}_l(s_t)\Delta t.
\end{equation}

Recall the boundary conditions $\gamma_l$, $\gamma_l(0)=0$, $\gamma_l(1)=1$, and the strictly increasing constraint $\dot{\gamma}_l>0$. In practice, the strictly increasing constraint is relaxed with a non-decreasing $\dot{\gamma}_l\geq 0$. This allows us to rewrite the boundary conditions in terms of the derivative:
\begin{equation}
    \int_{t=0}^1 \dot{\gamma}_l dt = \gamma_l(1) - \gamma_l(0) = 1.
\end{equation}

For its discrete analog, this is equivalent to:
\begin{equation}
    \sum_{t=0}^{T-1} v_l(t) = 1.
\end{equation}

Thus, each $v_l$ lies in the simplex
\begin{equation}
    \Delta^T
    =
    \left\{
    v\in\mathbb{R}_{\geq 0}^T:
    \sum_{t=0}^{T-1}v(t)=1
    \right\}.
\end{equation}

The corresponding discrete warping function is obtained by cumulative summation:
\begin{equation}
    \gamma_l^d(s_0)=0,
    \qquad
    \gamma_l^d(s_t)
    =
    \sum_{k=0}^{t-1}v_l(k),
    \qquad
    t=1,\dots,T.
\end{equation}

We represent the discretized mean SRVF by a Euclidean variable $v_{L+1}\in\mathbb{R}^T$.
Hence, the feasible set for the discretized registration problem is
\begin{equation}
    \mathcal{F}_{\mathrm{simplex}}
    =
    \underbrace{
    \Delta^T\times\cdots\times\Delta^T
    }_{L~\mathrm{times}}
    \times
    \mathbb{R}^T,
\end{equation}
where the first $L$ factors correspond to the warping-rate variables and the final factor corresponds to the discretized mean SRVF.

Therefore, Theorems~\ref{thm:simplex_1_smooth_1},
\ref{thm:simplex_2_smooth_2},
\ref{thm:smooth_2_simplex_1}, and~\ref{thm:smooth_2_simplex_2}
apply to this discretized SRVF registration problem and we compare PGD, RGD, PGD BB, and RGD BB for optimizing the resulting registration objective. For each algorithm, we update $v_l$ in each iteration, and subsequently determine the corresponding warping function $\gamma_l^d$. By composing an unaligned function with corresponding $\gamma_l^d$ using linear interpolation, we obtain the aligned function. The aligned functions are then used to update the mean $v_{L+1}$ in each iteration. Through this iterative process, we ultimately achieve the optimal $v_l$ and the corresponding aligned functions. Autograd in PyTorch is used to obtained gradient of the objective function.

\section{Numerical Results}
\label{results}
We benchmarked the following algorithms on a MacBook Air (Apple M3, 25GB RAM). All algorithms were implemented in Jax and PyTorch library in Python.

\begin{enumerate}
    \item RGD BB with backtracking update step \cite{iannazzo2018riemannian}, outlined in Algorithm~\ref{alg:riemannian_bb}.
    \item RGD with Armijo Condition based backtracking (Algorithm~\ref{alg:riemann}) as outlined in Li et al.~\cite{li2021simplex}, without the Wolfe condition.
    \item PGD with projection based on Duchi et al.~\cite{duchi2008efficient} outlined in Algorithm~\ref{alg:pgd}.
    \item PGD with Barzilai-Borwein step size (PGD BB) \cite{birgin2000nonmonotone} with projection based on Duchi et al.~\cite{duchi2008efficient}.
\end{enumerate}
\begin{algorithm}[!h]
\caption{PGD with Backtracking Update Step}\label{alg:pgd}
\begin{algorithmic}[1]
\Require initial step size $\eta>0$, initial reparameterized factor vector $\theta$, empirical tensor $\mathcal{Q}$, Armijo constant $c>0$, and backtracking decay factor $\beta\in(0,1)$.
\State $v = \theta \odot \theta$
\State $G_\mathrm{E} \gets \nabla_v f(v)$ \Comment{Euclidean Gradient}
\State $\hat{v} = \textrm{product\_simplex\_projection}(v -\eta*G_\mathrm{E})$
\State $\alpha \gets 1$
\State $m \gets 1$
\State $\textrm{ArmijoFlag} \gets \textrm{False}$
\While{\textrm{ArmijoFlag} = \textrm{False} and $m \leq 25$}
\State $v_{\textrm{new}} = v + \alpha*(\hat{v}-v)$
\If{$f(v_{\textrm{new}}) - f(v) <= c*\alpha*\langle G_\mathrm{E},  \hat{v} - v \rangle$}
    \State \textrm{ArmijoFlag} = \textrm{True}  
\Else{}
    \State $\alpha \gets \beta \times \alpha$
\EndIf
\State $m \gets m+1$
\EndWhile
\State \Return{$v_{\textrm{new}}$} 
\end{algorithmic}
\end{algorithm}
\subsection{Simplex Constrained Canonical Polyadic Tensor Factorization}
In this section, we present results from numerical simulations to demonstrate the performance of our algorithms. We generate a random joint probability distribution with three random variables, $P(X_1, X_2, X_3)$. The data is generated from a Naive Bayes CPD Model:
\[
P(X_1, X_2, X_3) = \sum_{\phi=1}^F P(X_1|t) P(X_2|t) P(X_3|t) P(t).
\]
The number of values of each variable $X_n$ is chosen from the set $\{4, 8, 16\}$. The rank $F$ of the CPD is selected from $\{4, 8, 16\}$. For all algorithms, we test initial step sizes ($\eta$) of the form $2^i$, where $i$ ranges between $[-4, 4]$. All algorithms are run for a combination of these parameters until a maximum time of 20 seconds is reached. The decay factor $\beta$ is set to 0.5 for both RGD BB and PGD BB, and 0.75 for RGD and PGD. A small decay constant $c = 10^{-5}$ is chosen for the BB algorithms, while $c = 10^{-4}$ is used for PGD and RGD. These parameters were hand-tuned to achieve optimal performance for each algorithm. The value $m \leq 25$ in line~\ref{line:5} of Algorithm~\ref{alg:riemann} comes from~\cite{li2021simplex}.

\begin{figure*}[t]
    \centering
    \includegraphics[width=0.33\textwidth]{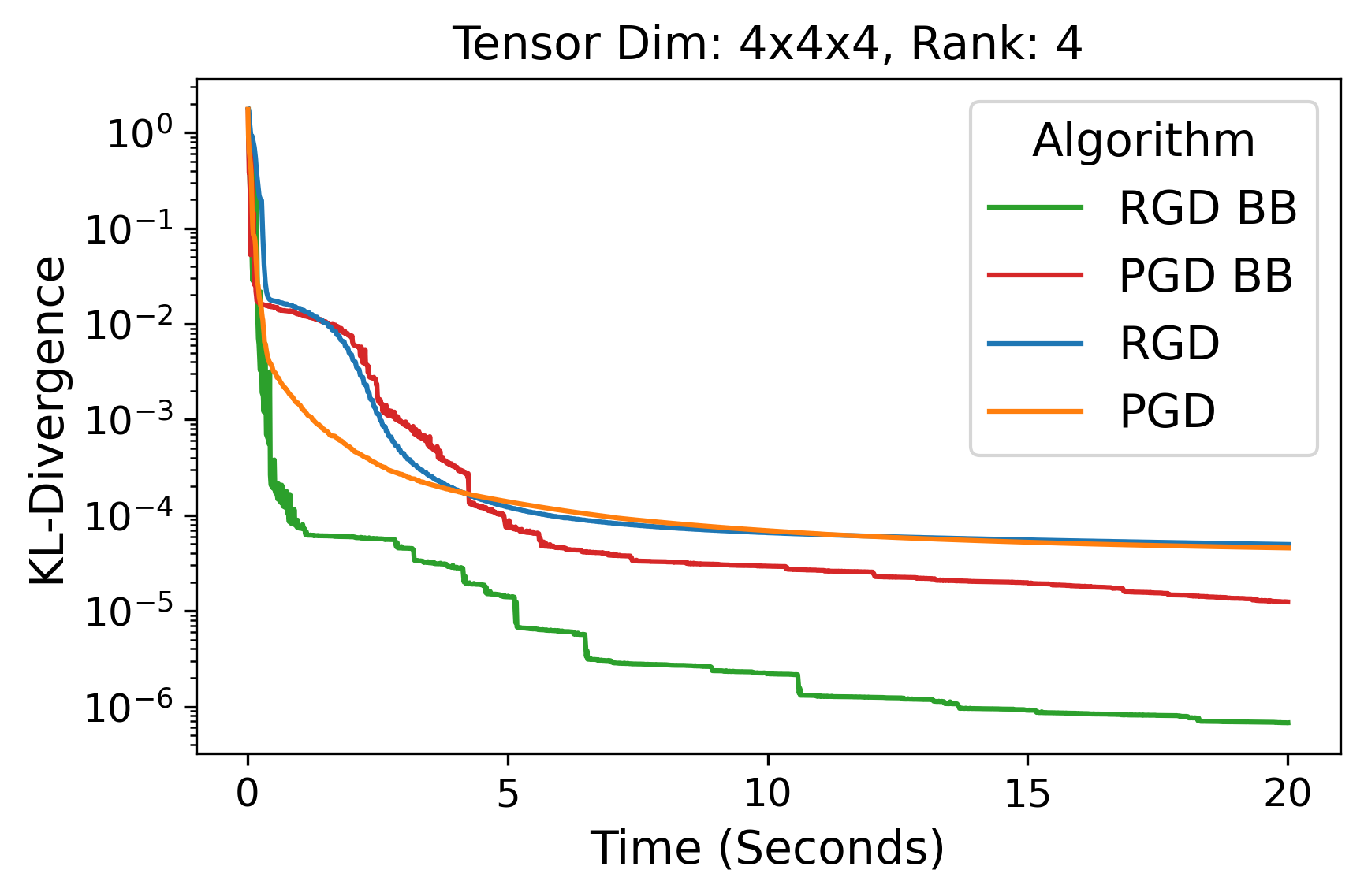}%
    \includegraphics[width=0.33\textwidth]{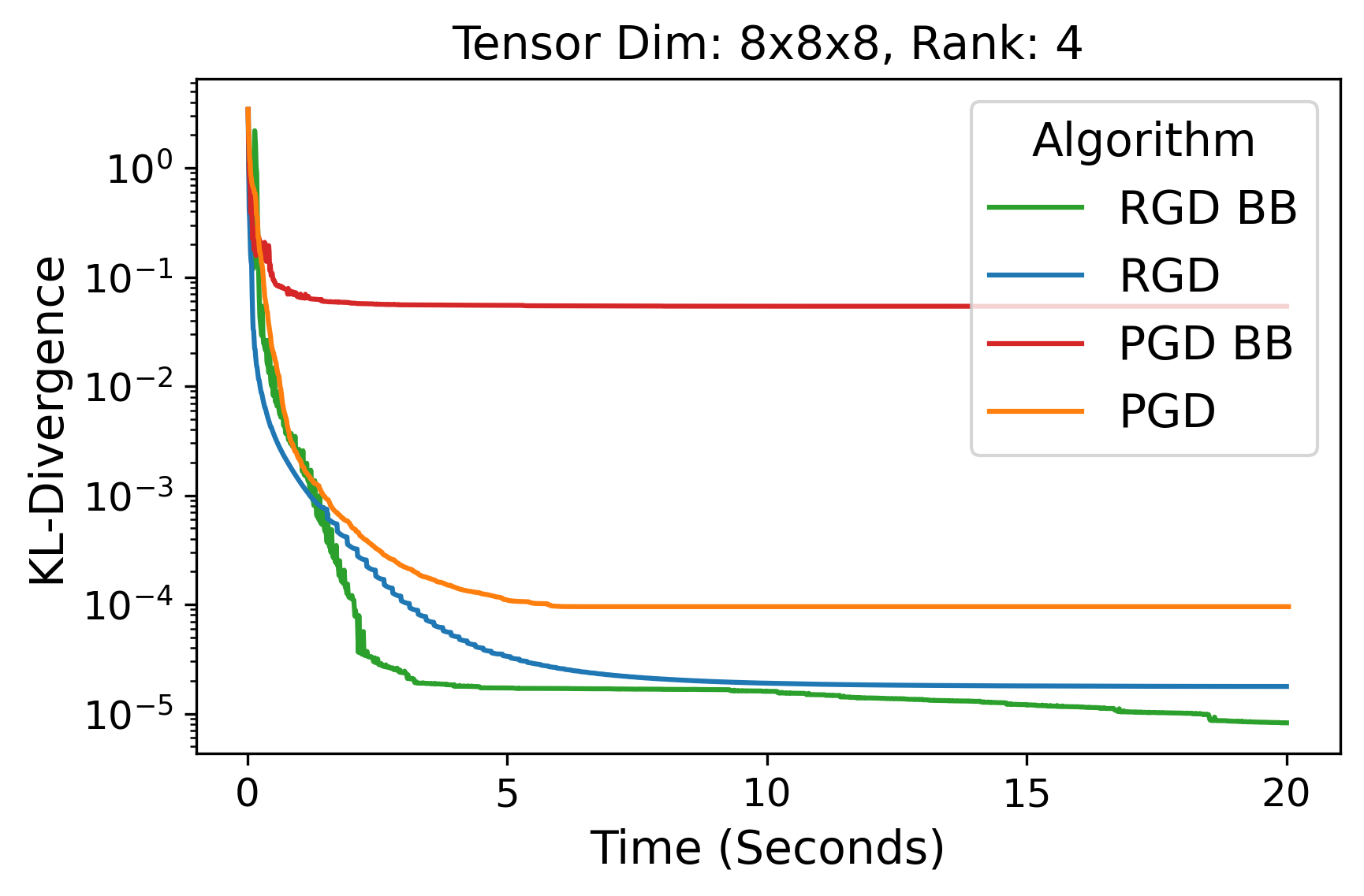}%
    \includegraphics[width=0.33\textwidth]{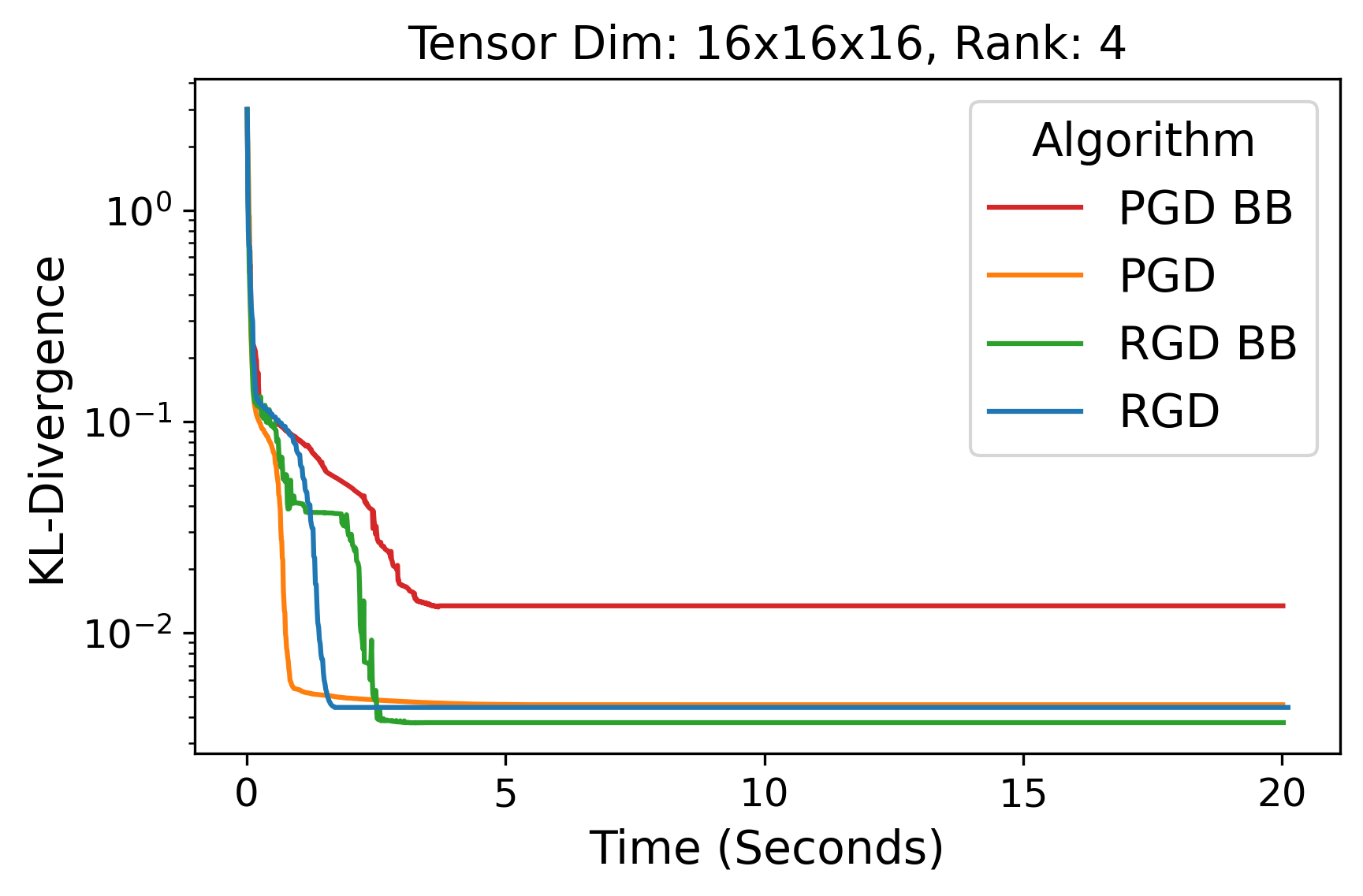}%
    
    \includegraphics[width=0.33\textwidth]{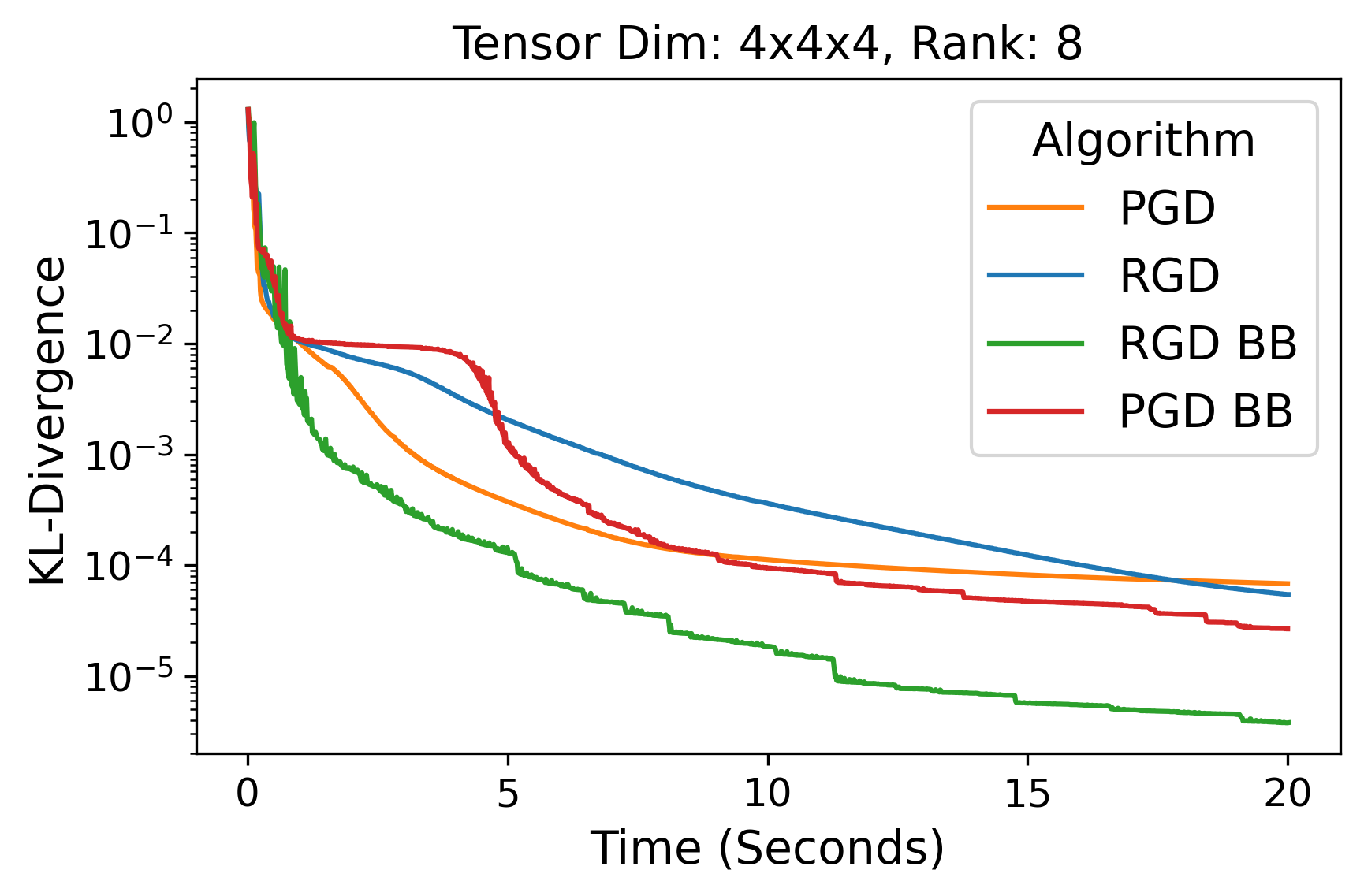}%
    \includegraphics[width=0.33\textwidth]{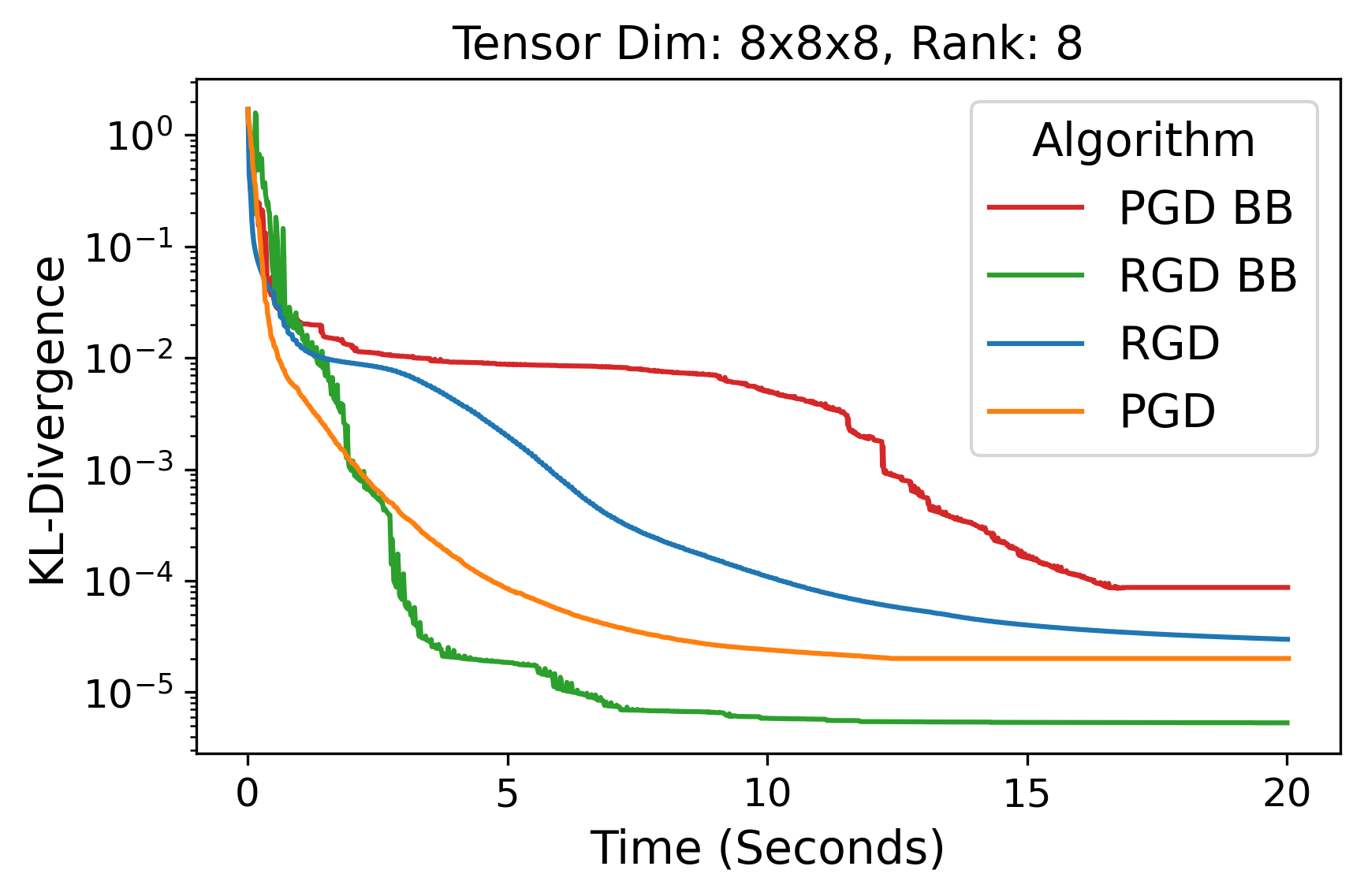}%
    \includegraphics[width=0.33\textwidth]{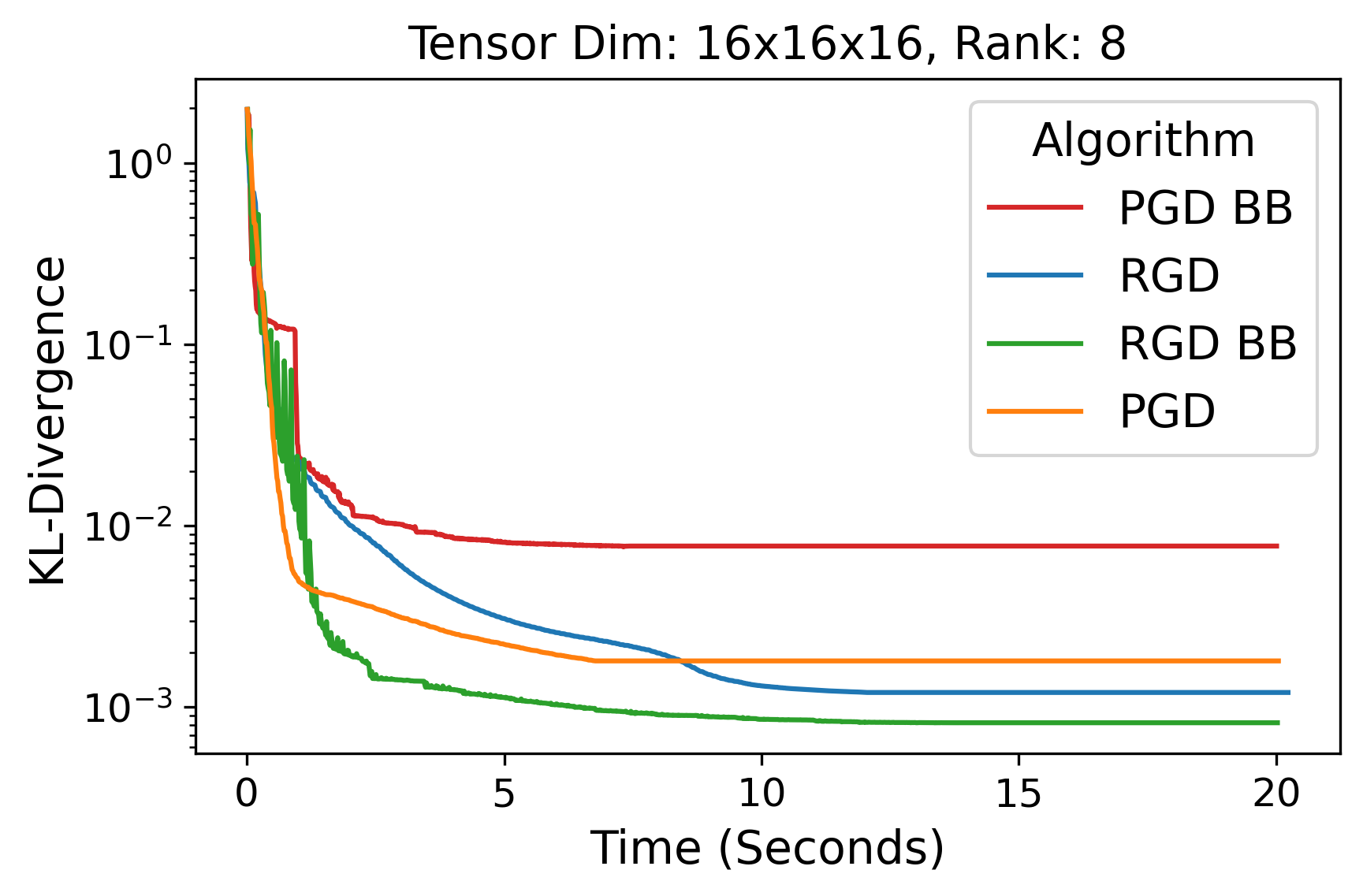}%
    
    \includegraphics[width=0.33\textwidth]{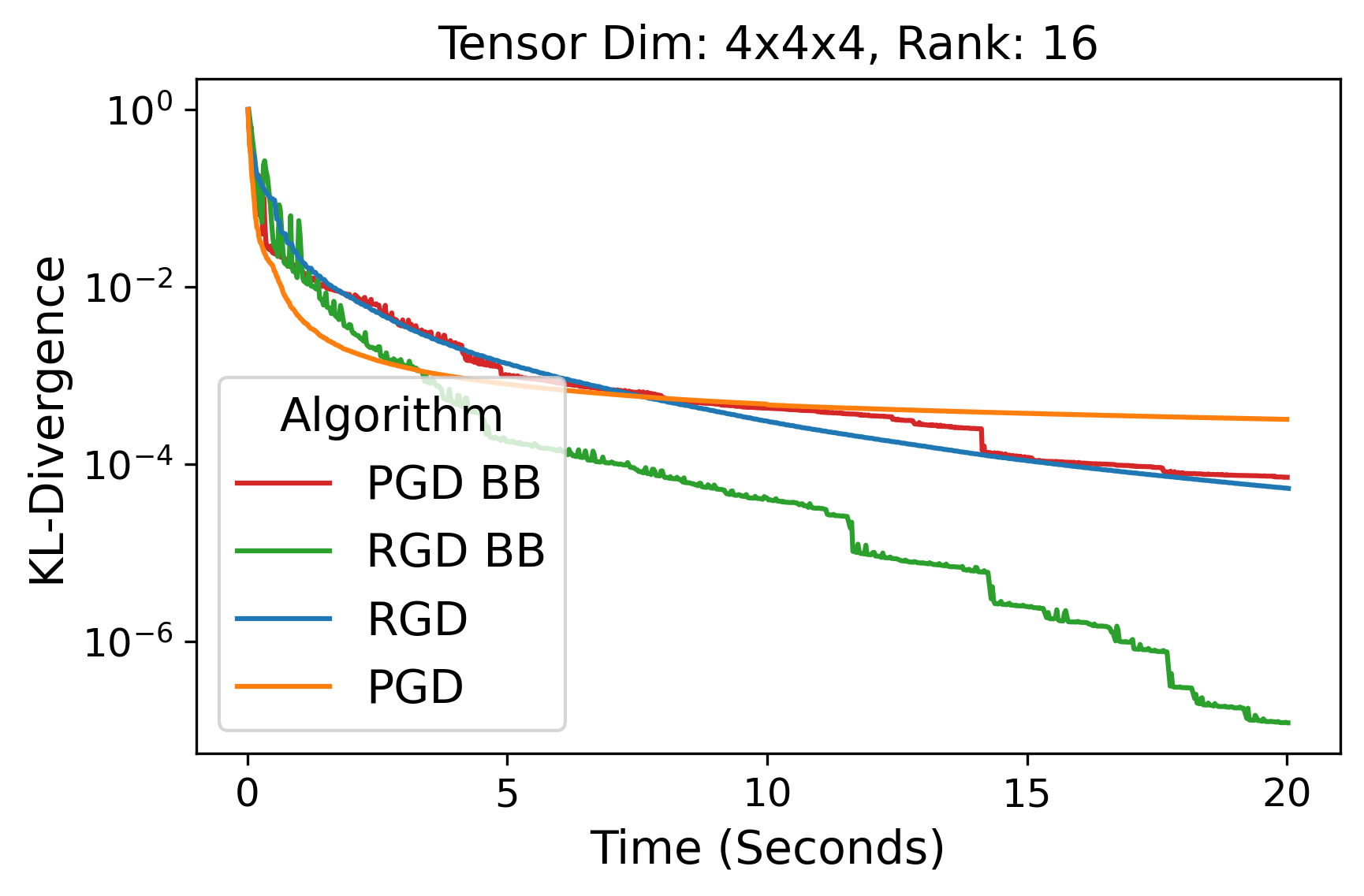}%
    \includegraphics[width=0.33\textwidth]{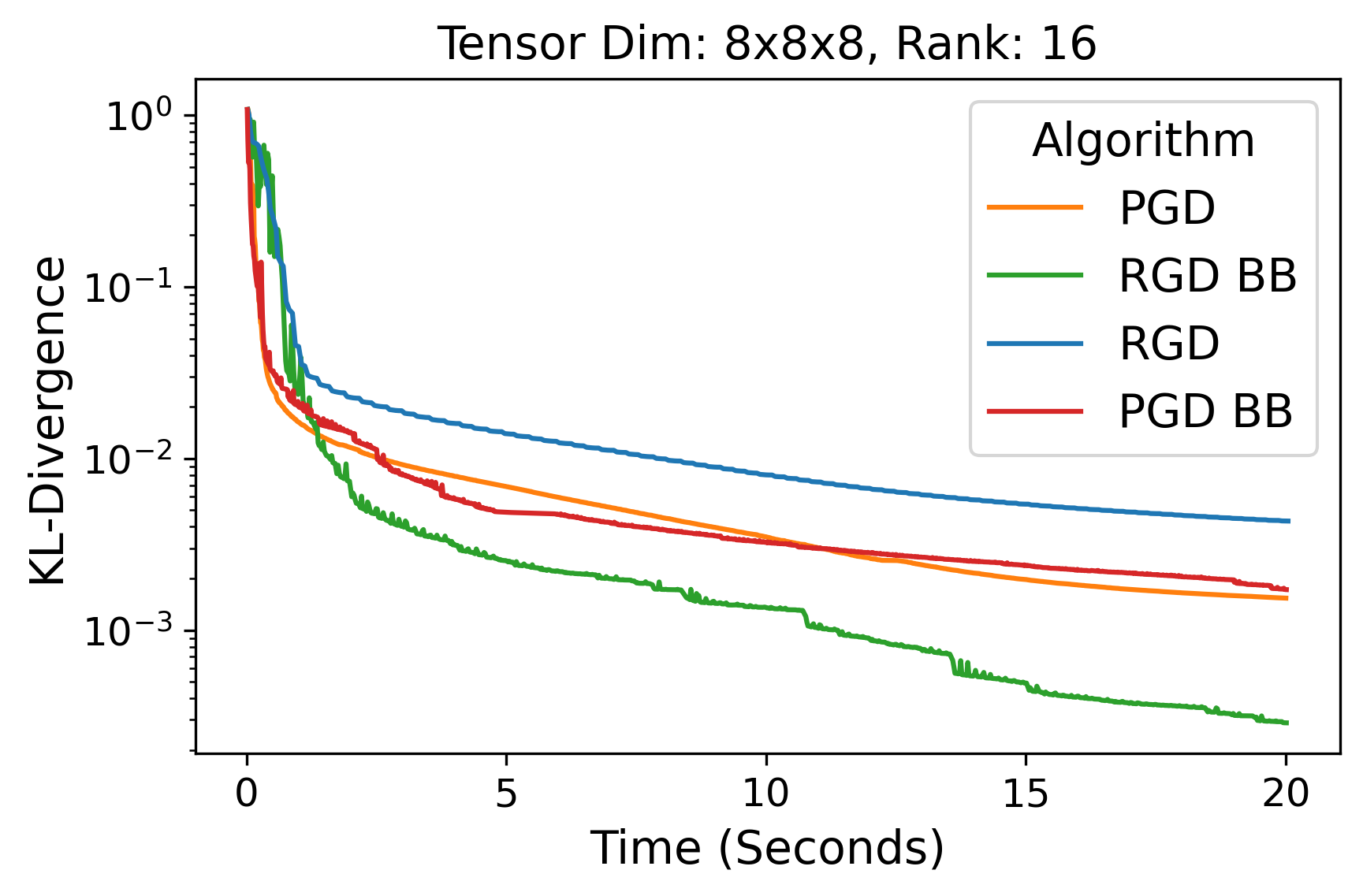}%
    \includegraphics[width=0.33\textwidth]{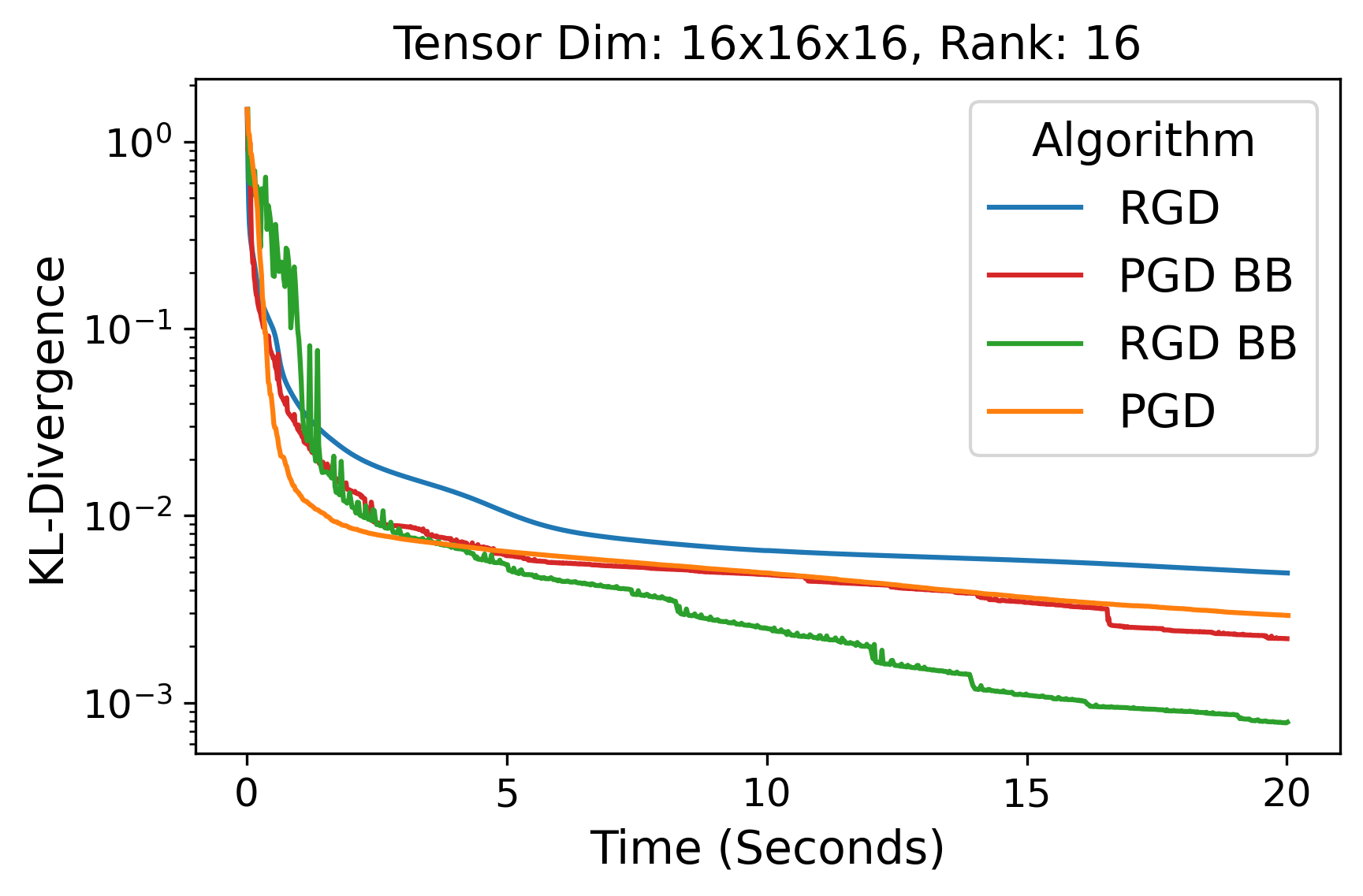}%
    \caption{For each algorithm, we choose best step size $\eta$ for different rank and tensor size. RGD BB outperforms other algorithm, often by several orders of magnitude. Higher-dimensional tensors are harder to factorize than lower-dimensional ones, although rank of tensor does not seem to affect the final quality of the solution as much.}
    \label{fig:tensor_performance}
\end{figure*}
In Fig.~\ref{fig:tensor_performance}, we report the results of the best performing initial step size $\eta$ for each algorithm across a combination of tensor sizes and ranks. As seen in most rank and tensor size combinations, RGD BB performs the best, often outperforming the other algorithms by several orders of magnitude in the final solution cost. This improvement appears to be primarily due to RGD BB requiring fewer iterations to reach low-cost solutions. Interestingly, while higher-dimensional tensors were harder to factorize, as evidenced by the lower KL-Divergence, we did not observe a similar effect with the rank.

In our Appendix \ref{appendix_b}, we report results on different initial step sizes $\eta$ for each algorithm. Out of all the algorithms, RGD BB often performs the best across a range of step sizes, being less sensitive to its misspecification. This makes it particularly attractive as a plug-and-play algorithm for large-scale problems, where extensive grid search over parameters is infeasible.

\begin{figure}[!b]
    \centering
    \includegraphics[width=1\linewidth]{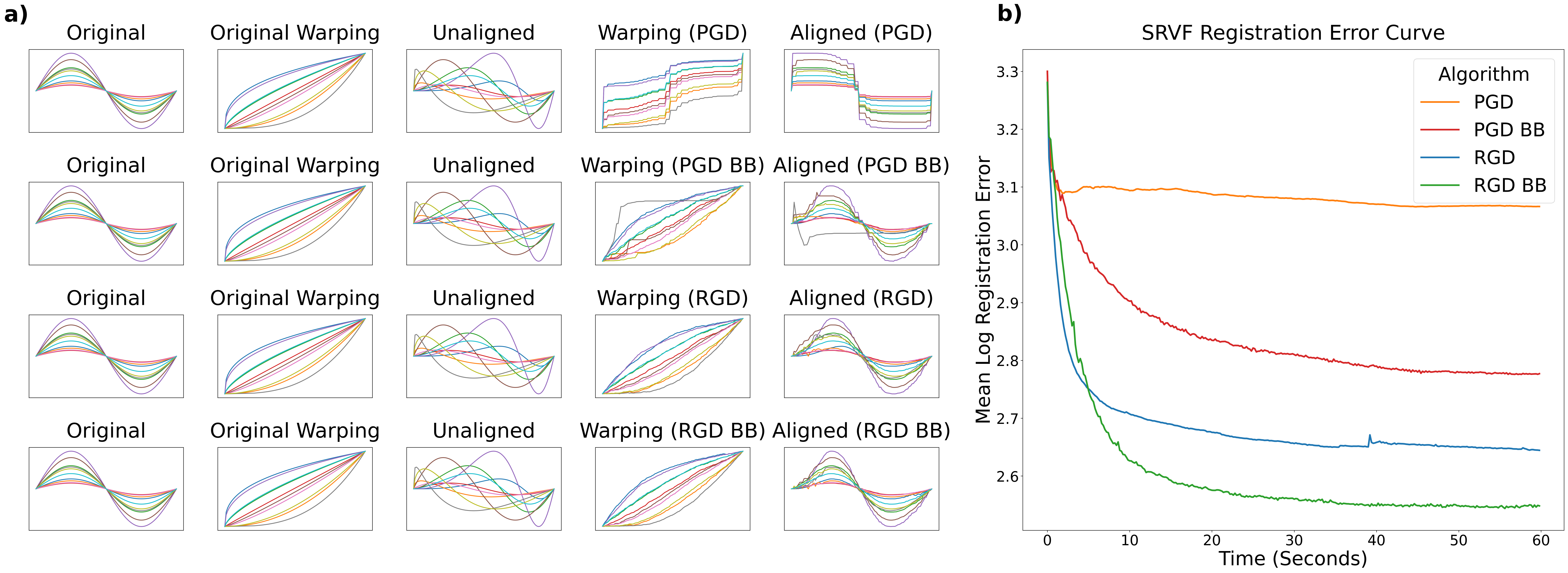}
    \caption{Comparison of algorithms in SRVF registration: a) Starting from unaligned functions in third column, RGD and RGD BB results in much better alignment and reconstruction of original sinousoidal function $\hat{\beta_l}$ shapes whereas PGD and PGD BB results in a distorted waveform. b) SRVF registration error curve of RGD and RGD BB compared to PGD and PGD BB.}
    \label{fig:srvf_result}
\end{figure}

\subsection{Functional Data Registration}
We start by randomly generating 10 sinusoidal functions $\beta_l$ with varying amplitudes (Fig.~\ref{fig:srvf_result}(a), first column). These could, for example, indicate a population of $L$ individuals performing a periodic motion (e.g., doing a bicep curl) at different speeds. These amplitude functions are then composed with some random diffeomorphic warpings of time $\gamma$. The result is unaligned functions in the third column of Fig.~\ref{fig:srvf_result}(a), which can synthesize the effect of population level variability across individuals performing the same motion at different rates. Our goal is to recover the original sinusoidal shape from these temporally distorted motions.

We parameterize the problem on the simplex in terms of $v$ and also reparameterize it smoothly in terms of $\theta$ under $\psi=\theta \odot \theta$. The warping functions $v_l$ are randomly intialized on the simplex ensuring $v_l\geq0$ and $\mathbf{1}^T v_l = 1$. The Euclidean mean SRVF parameter is stored in $v_{L+1}$. For all algorithms, we alternate between the following two steps: (a) computing the optimal warping function $\gamma_l^d$ via optimization of $v_l$, and (b) composing the unaligned functions with the warping functions to obtain the aligned functions and updating the SRVF mean $v_{L+1}$. In this manner, all algorithms are run for 60 seconds per trial, and the reported results are averaged over 50 trials. To determine the optimal learning rate, we try step sizes in $2^i$, $i$ ranging between $[-4, 4]$. Ultimately, we selected $0.5$ as our choice of step size as it yielded the best results. For all the algorithms, the decay factor $\beta$ is set to 0.5 and decay constant is chosen as $c = 10^{-4}$. Similar to the tensor decomposition problem, these parameters were hand-tuned to achieve optimal performance for each algorithm.

Fig.~\ref{fig:srvf_result}(a) illustrates the qualitative alignment results. The RGD and RGD BB warping functions are smoother and align more closely with the original compared to those produced by PGD and PGD BB (Fig.~\ref{fig:srvf_result}(a), fourth column). As a result, the RGD and RGD BB aligned functions remain closer to the original sine waves (Fig.~\ref{fig:srvf_result}(a), fifth column). In contrast, the PGD and PGD BB warping functions display irregularities and roughness, leading to aligned functions that deviate from the original shapes. Specifically, there are abrupt changes at the edges in PGD, and as a result, the sine waves transform into square waves. This is likely due to the projection operation in PGD resulting in $\dot{\gamma}$ ending up with solutions with zeros, where the reconstructed $\gamma$ look much more step like. However, in the reparameterized problem, the natural smoothness of sphere results in smoother warping functions, which are critical in applications requiring accurate diffeomorphic shape registration. 

Fig.~\ref{fig:srvf_result}(b) presents the mean log registration error curves for all the approaches, where the error corresponds to the value of the SRVF registration objective function. It demonstrates that RGD and RGD BB achieve faster convergence and more effective optimization, resulting in lower error compared to PGD and PGD BB. These findings emphasize the advantages of Riemannian techniques in achieving efficient and reliable performance in aligning functional data.
\section{Conclusion}
In this work, we presented a framework for smooth reparameterization of $L$ vectors on simplicial product spaces. We showed that reparameterization preserves weak second-order KKT points and leveraged an efficient Riemannian Barzilai-Borwein algorithm to address the reparameterized problem. The Riemannian algorithm achieves excellent quantitative performance on both problems mentioned in section~\ref{results} and is robust to the choice of initial step size.
\section*{Funding}
This work was supported by the Fonds de la Recherche Scientifique -- FNRS under Grant no T.0001.23.
\section*{Competing Interests}
The authors declare that they have no competing interests.
\section*{Code Availability}
The implementation used to generate all numerical experiments in this paper is available at~\url{https://github.com/Arafat245/smooth_reparameterizations}.
\bibliographystyle{spmpsci}
\bibliography{references}

@article{kizel2017stepwise,
  title={A stepwise analytical projected gradient descent search for hyperspectral unmixing and its code vectorization},
  author={Kizel, Fadi and Shoshany, Maxim and Netanyahu, Nathan S and Even-Tzur, Gilad and Benediktsson, J{\'o}n Atli},
  journal={IEEE Transactions on Geoscience and Remote Sensing},
  volume={55},
  number={9},
  pages={4925--4943},
  year={2017},
  publisher={IEEE}
}

@inproceedings{wu2017stochastic,
  title={A stochastic maximum-likelihood framework for simplex structured matrix factorization},
  author={Wu, Ruiyuan and Ma, Wing-Kin and Fu, Xiao},
  booktitle={2017 IEEE International Conference on Acoustics, Speech and Signal Processing (ICASSP)},
  pages={2557--2561},
  year={2017},
  organization={IEEE}
}

@article{ma2013signal,
  title={A signal processing perspective on hyperspectral unmixing: Insights from remote sensing},
  author={Ma, Wing-Kin and Bioucas-Dias, Jos{\'e} M and Chan, Tsung-Han and Gillis, Nicolas and Gader, Paul and Plaza, Antonio J and Ambikapathi, ArulMurugan and Chi, Chong-Yung},
  journal={IEEE Signal Processing Magazine},
  volume={31},
  number={1},
  pages={67--81},
  year={2013},
  publisher={IEEE}
}

@article{abdolali2021simplex,
  title={Simplex-structured matrix factorization: Sparsity-based identifiability and provably correct algorithms},
  author={Abdolali, Maryam and Gillis, Nicolas},
  journal={SIAM Journal on Mathematics of Data Science},
  volume={3},
  number={2},
  pages={593--623},
  year={2021},
  publisher={SIAM}
}

@article{sidiropoulos2023minimizing,
  title={Minimizing low-rank models of high-order tensors: Hardness, span, tight relaxation, and applications},
  author={Sidiropoulos, Nicholas D and Karakasis, Paris A and Konar, Aritra},
  journal={IEEE Transactions on Signal Processing},
  year={2023},
  publisher={IEEE}
}

@article{bertsekas1997nonlinear,
  title={Nonlinear programming},
  author={Bertsekas, Dimitri P},
  journal={Journal of the Operational Research Society},
  volume={48},
  number={3},
  pages={334--334},
  year={1997},
  publisher={Taylor \& Francis}
}

@book{gill2019practical,
  title={Practical optimization},
  author={Gill, Philip E and Murray, Walter and Wright, Margaret H},
  year={2019},
  publisher={SIAM}
}

@book{bonnans2006numerical,
  title={Numerical optimization: theoretical and practical aspects},
  author={Bonnans, Joseph-Fr{\'e}d{\'e}ric and Gilbert, Jean Charles and Lemar{\'e}chal, Claude and Sagastiz{\'a}bal, Claudia A},
  year={2006},
  publisher={Springer Science \& Business Media}
}

@book{nocedal2006numerical,
  author = {Jorge Nocedal and Stephen J. Wright},
  title = {Numerical Optimization},
  edition = {2nd},
  year = {2006},
  publisher = {Springer},
}

@book{aragon2019nonlinear,
  author = {Francisco J. Aragón and others},
  title = {Nonlinear Optimization},
  year = {2019},
  publisher = {Society for Industrial and Applied Mathematics},
}

@article{faybusovich1991hamiltonian,
  title={Hamiltonian structure of dynamical systems which solve linear programming problems},
  author={Faybusovich, Leonid},
  journal={Physica D: Nonlinear Phenomena},
  volume={53},
  number={2-4},
  pages={217--232},
  year={1991},
  publisher={Elsevier}
}

@article{kargas2018tensors,
  title={Tensors, learning, and “{Kolmogorov} extension” for finite-alphabet random vectors},
  author={Kargas, Nikos and Sidiropoulos, Nicholas D and Fu, Xiao},
  journal={IEEE Transactions on Signal Processing},
  volume={66},
  number={18},
  pages={4854--4868},
  year={2018},
  publisher={IEEE}
}

@article{li2021simplex,
  title={From the simplex to the sphere: Faster constrained optimization using the {Hadamard} parametrization},
  author={Li, Qiuwei and McKenzie, Daniel and Yin, Wotao},
  journal={arXiv preprint arXiv:2112.05273},
  year={2021}
}

@inproceedings{yeredor2019estimation,
  title={Estimation of a low-rank probability-tensor from sample sub-tensors via joint factorization minimizing the {Kullback-Leibler} divergence},
  author={Yeredor, Arie and Haardt, Martin},
  booktitle={2019 27th European Signal Processing Conference (EUSIPCO)},
  pages={1--5},
  year={2019},
  organization={IEEE}
}

@article{levin2022effect,
  title={The effect of smooth parametrizations on nonconvex optimization landscapes},
  author={Levin, Eitan and Kileel, Joe and Boumal, Nicolas},
  journal={arXiv preprint arXiv:2207.03512},
  year={2022}
}

@article{esposito2025riemannian,
  title={Riemannian Multiplicative Update for Sparse Simplex constraint using oblique rotation manifold},
  author={Esposito, Flavia and Ang, Andersen},
  journal={arXiv preprint arXiv:2503.24075},
  year={2025}
}

@article{guo2021sparse,
  title={A sparse oblique-manifold nonnegative matrix factorization for hyperspectral unmixing},
  author={Guo, Ziyang and Min, Anyou and Yang, Bing and Chen, Junhong and Li, Hong and Gao, Junbin},
  journal={IEEE Transactions on Geoscience and Remote Sensing},
  volume={60},
  pages={1--13},
  year={2021},
  publisher={IEEE}
}

@article{vaskevicius2019implicit,
  title={Implicit regularization for optimal sparse recovery},
  author={Vaskevicius, Tomas and Kanade, Varun and Rebeschini, Patrick},
  journal={Advances in Neural Information Processing Systems},
  volume={32},
  year={2019}
}

@inproceedings{duchi2008efficient,
  title={Efficient projections onto the $l 1$-ball for learning in high dimensions},
  author={Duchi, John and Shalev-Shwartz, Shai and Singer, Yoram and Chandra, Tushar},
  booktitle={Proceedings of the 25th international conference on Machine learning},
  pages={272--279},
  year={2008}
}

@book{srivastava2016functional,
  title={Functional and shape data analysis},
  author={Srivastava, Anuj and Klassen, Eric P},
  volume={1},
  year={2016},
  publisher={Springer}
}

@article{tucker2013generative,
  title={Generative models for functional data using phase and amplitude separation},
  author={Tucker, J Derek and Wu, Wei and Srivastava, Anuj},
  journal={Computational Statistics \& Data Analysis},
  volume={61},
  pages={50--66},
  year={2013},
  publisher={Elsevier}
}

@article{sidiropoulos2017tensor,
  title={Tensor decomposition for signal processing and machine learning},
  author={Sidiropoulos, Nicholas D and De Lathauwer, Lieven and Fu, Xiao and Huang, Kejun and Papalexakis, Evangelos E and Faloutsos, Christos},
  journal={IEEE Transactions on signal processing},
  volume={65},
  number={13},
  pages={3551--3582},
  year={2017},
  publisher={IEEE}
}

@article{amiridi2021information,
  title={Information-theoretic feature selection via tensor decomposition and submodularity},
  author={Amiridi, Magda and Kargas, Nikos and Sidiropoulos, Nicholas D},
  journal={IEEE Transactions on Signal Processing},
  volume={69},
  pages={6195--6205},
  year={2021},
  publisher={IEEE}
}

@article{birgin2000nonmonotone,
  title={Nonmonotone spectral projected gradient methods on convex sets},
  author={Birgin, Ernesto G and Mart{\'\i}nez, Jos{\'e} Mario and Raydan, Marcos},
  journal={SIAM Journal on Optimization},
  volume={10},
  number={4},
  pages={1196--1211},
  year={2000},
  publisher={SIAM}
}

@article{iannazzo2018riemannian,
  title={The {Riemannian} {Barzilai--Borwein} method with nonmonotone line search and the matrix geometric mean computation},
  author={Iannazzo, Bruno and Porcelli, Margherita},
  journal={IMA Journal of Numerical Analysis},
  volume={38},
  number={1},
  pages={495--517},
  year={2018},
  publisher={Oxford University Press}
}

@article{zhao2019implicit,
  title={Implicit regularization via {Hadamard} product over-parametrization in high-dimensional linear regression},
  author={Zhao, Peng and Yang, Yun and He, Qiao-Chu},
  journal={arXiv preprint arXiv:1903.09367},
  volume={2},
  number={4},
  pages={8},
  year={2019}
}

@article{kumar2025shape,
  title={A shape-based functional index for objective assessment of pediatric motor function},
  author={Kumar, Shashwat and Rahman, Arafat and Gutierrez, Robert and Livermon, Sarah and McCrady, Allison N and Blemker, Silvia and Scharf, Rebecca and Srivastava, Anuj and Barnes, Laura E},
  journal={Plos one},
  volume={20},
  number={10},
  pages={e0332383},
  year={2025},
  publisher={Public Library of Science San Francisco, CA USA}
}

@BOOK{AMS2008,
 author = "P.-A. Absil and R. Mahony and R. Sepulchre",
 title = "Optimization Algorithms on Matrix Manifolds",
 publisher = "Princeton University Press",
 address = "Princeton, NJ",
 year = 2008,
 month_hide = "January",
 isbn = "978-0-691-13298-3",
 isbn-10 = "0-691-13298-4",
 url = "http://press.princeton.edu/titles/8586.html",
 }

@article{AbsMal2012,
author = {P.-A. Absil and J\'{e}r\^{o}me Malick},
title = {Projection-like Retractions on Matrix Manifolds},
publisher = {SIAM},
year = {2012},
journal = {SIAM Journal on Optimization},
volume = {22},
number = {1},
pages = {135-158},
url = {http://link.aip.org/link/?SJE/22/135/1},
doi = {10.1137/100802529}
}

@Book{boumal2023intromanifolds,
  title     = {An introduction to optimization on smooth manifolds},
  author    = {Boumal, Nicolas},
  publisher = {Cambridge University Press},
  year      = {2023},
  url       = {https://www.nicolasboumal.net/book},
  doi       = {10.1017/9781009166164},
}

@article{GuoLinYe2013,
author = "Guo, Lei and Lin, Gui-Hua and Ye, Jane J.",
title = "Second-Order Optimality Conditions for Mathematical Programs with Equilibrium Constraints",
journal = "Journal of Optimization Theory and Applications",
year = 2013,
volume = 158,
pages = "33--64",
doi = "10.1007/s10957-012-0228-x",
}

@book{niculescu2006convex,
  title={Convex functions and their applications},
  author={Niculescu, Constantin and Persson, Lars-Erik},
  volume={23},
  year={2006},
  publisher={Springer}
}

\newpage
\onecolumn
\appendix

\newpage
\section{Appendix (Results with Different Learning Rates)}
\label{appendix_b}

\begin{figure*}[h]
    \centering
    \includegraphics[width=0.99\textwidth]{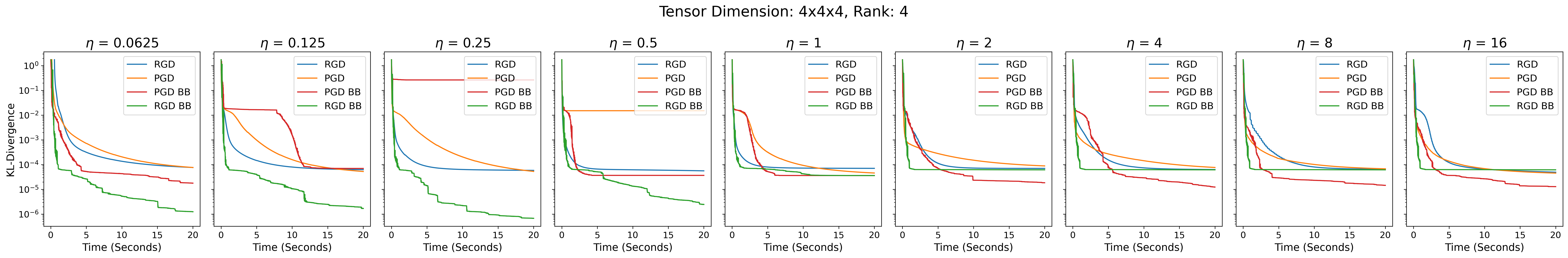}
    \includegraphics[width=0.99\textwidth]{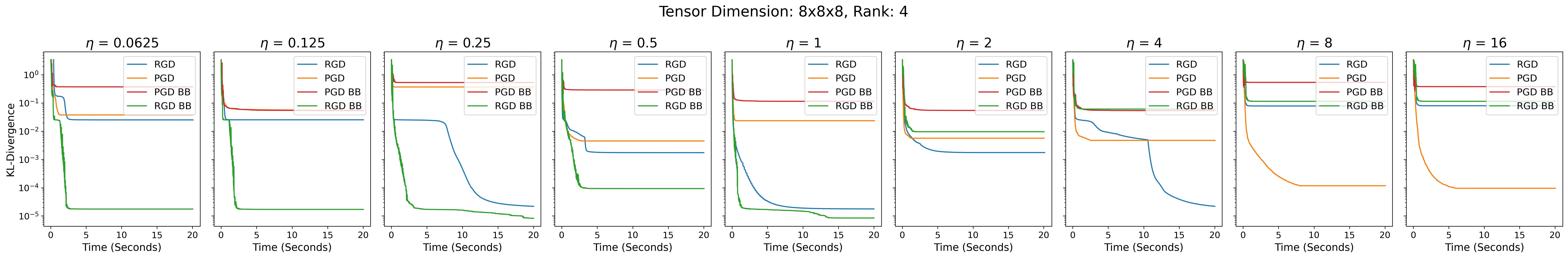}
    \includegraphics[width=0.99\textwidth]{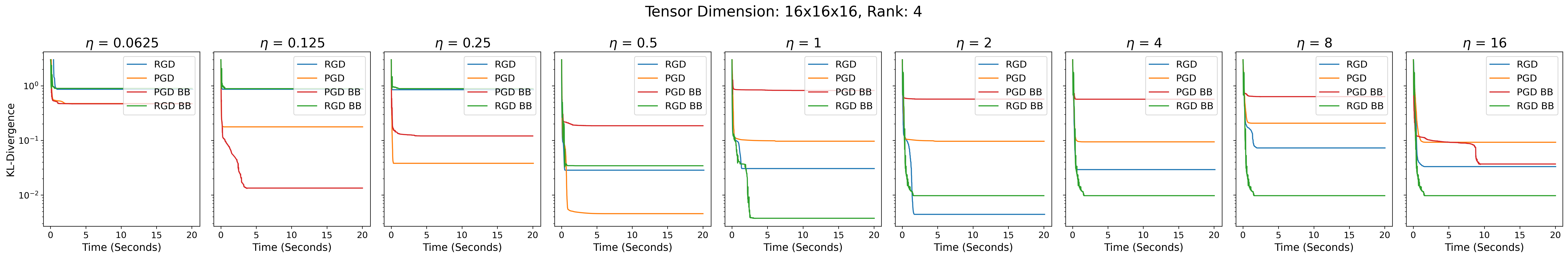}
    \includegraphics[width=0.99\textwidth]{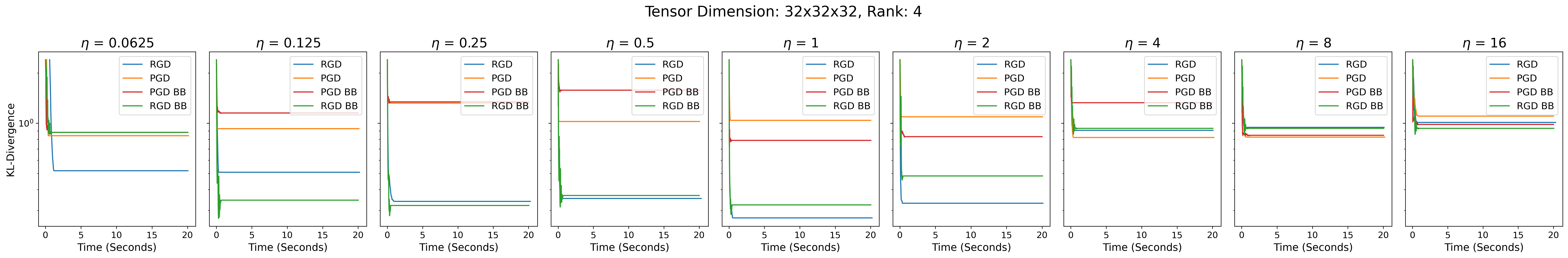}
    \caption{Results for tensor dimensions 4, 8, 16, 32 of rank 4 with different learning rates $\eta$.}
    \label{fig:tensor_rank4}
\end{figure*}

\begin{figure*}[h]
    \centering
    \includegraphics[width=0.99\textwidth]{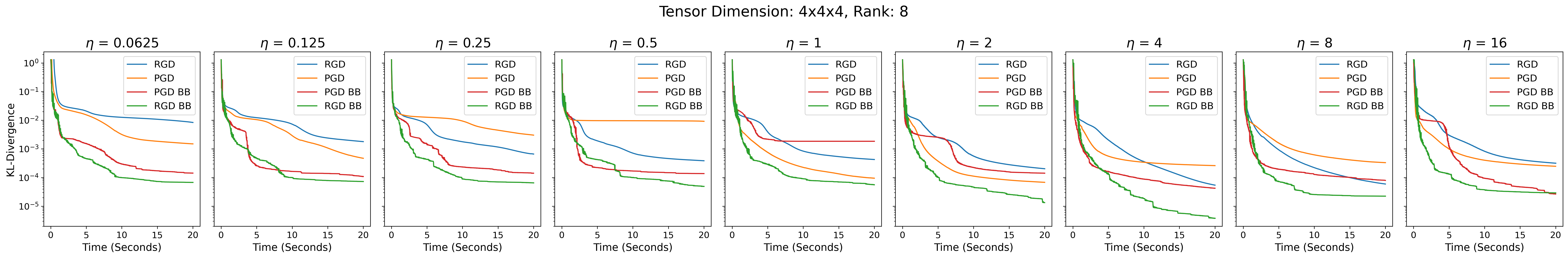}
    \includegraphics[width=0.99\textwidth]{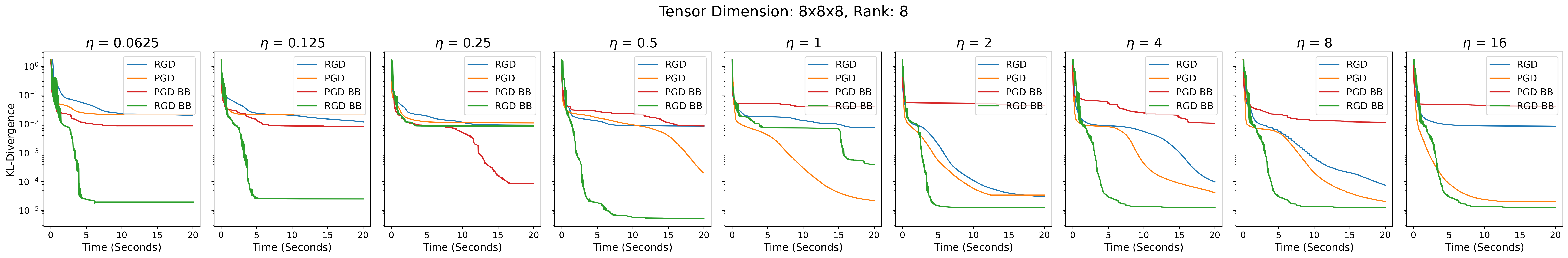}
    \includegraphics[width=0.99\textwidth]{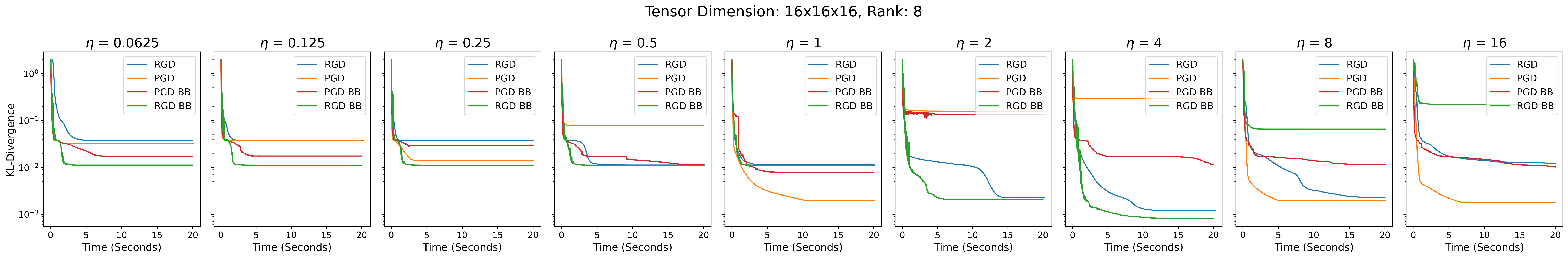}
    \includegraphics[width=0.99\textwidth]{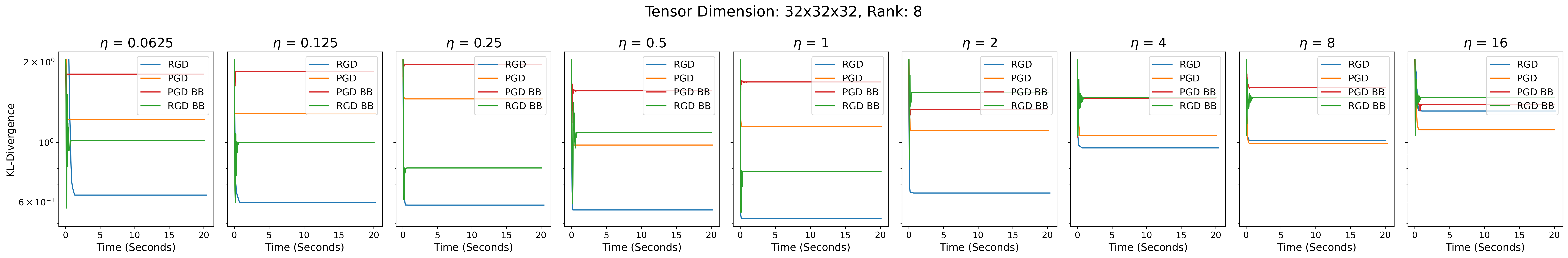}
    \caption{Results for tensor dimensions 4, 8, 16, 32 of rank 8 with different learning rates $\eta$.}
    \label{fig:tensor_rank8}
    \end{figure*}

    \begin{figure*}[h]
    \centering
    \includegraphics[width=0.99\textwidth]{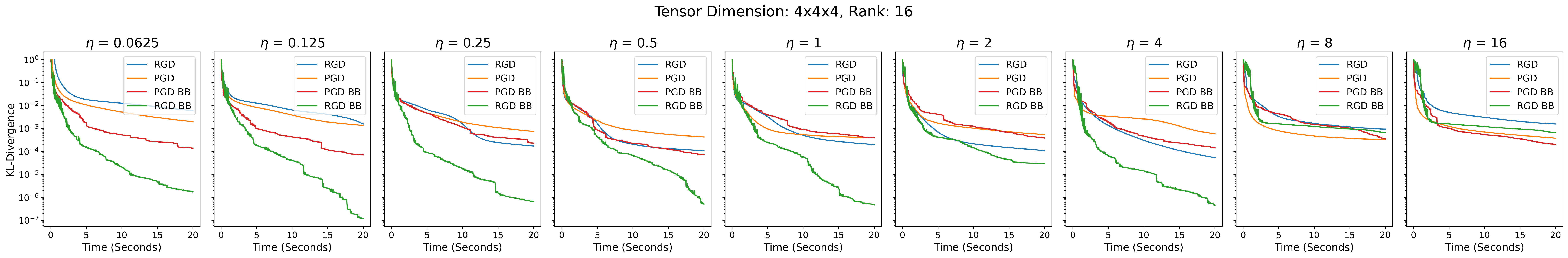}
    \includegraphics[width=0.99\textwidth]{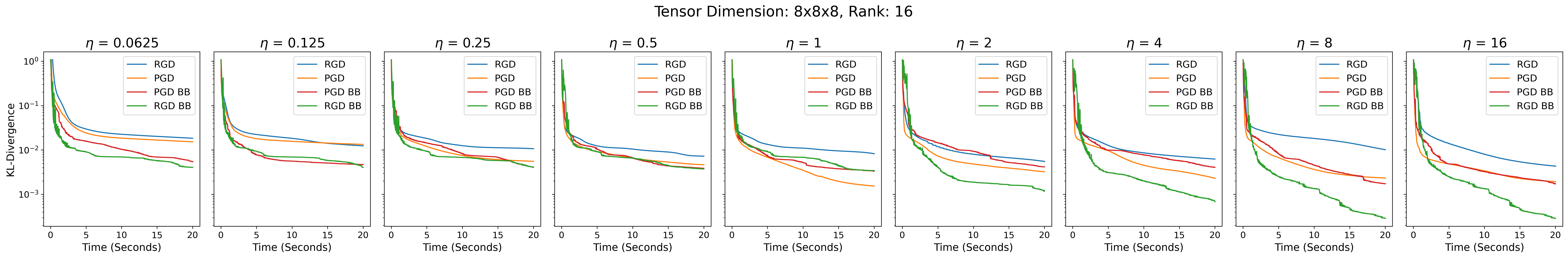}
    \includegraphics[width=0.99\textwidth]{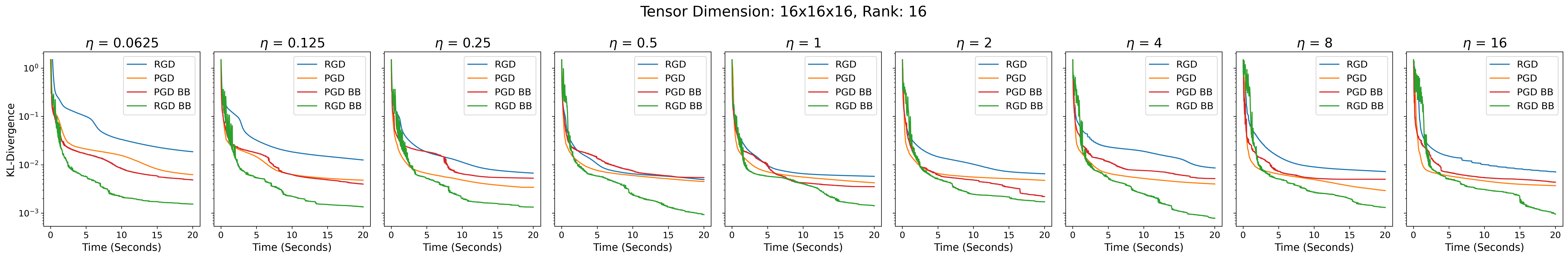}
    \includegraphics[width=0.99\textwidth]{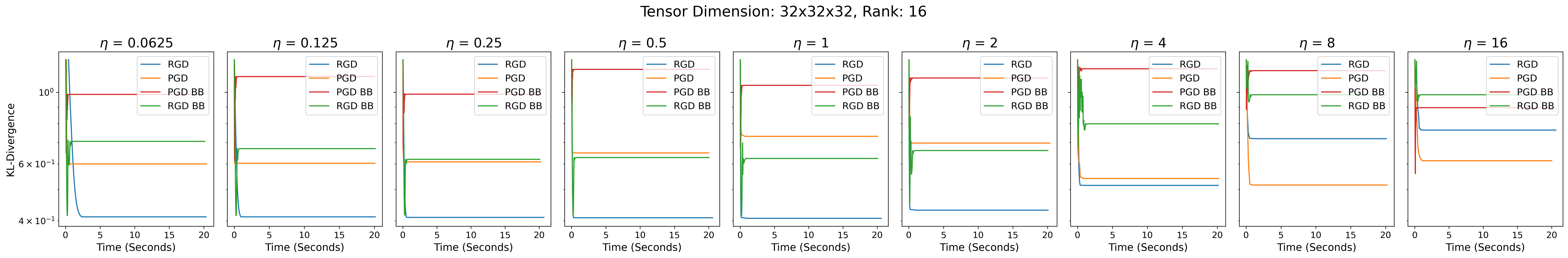}
    \caption{Results for tensor dimensions 4, 8, 16, 32 of rank 16 with different learning rates $\eta$.}
    \label{fig:tensor_rank16}
    \end{figure*}

\end{document}